\documentclass[11pt]{article}

\usepackage{authblk}
\author[1]{Ingvar Ziemann\thanks{Corresponding author: ingvarz@seas.upenn.edu}}
\author[2]{Stephen Tu}
\author[1]{George J. Pappas}
\author[1]{Nikolai Matni}

\affil[1]{University of Pennsylvania}
\affil[2]{University of Southern California}

\usepackage[utf8]{inputenc} 
\usepackage[T1]{fontenc} 

\usepackage{fullpage}

\usepackage{url}            
\usepackage{booktabs}       
\usepackage{amsfonts}       
\usepackage{nicefrac}       
\usepackage{microtype}      
\usepackage[dvipsnames]{xcolor}

\usepackage{natbib} 
\usepackage{enumitem}
\usepackage{amsthm}
\usepackage{amssymb}
\usepackage{amsmath}
\usepackage{mathrsfs}
\usepackage{thmtools}
\usepackage{hyperref}
\hypersetup{
    pdftoolbar=true,        
    pdfmenubar=true,        
    pdffitwindow=false,     
    pdfstartview={FitH},    
    pdfnewwindow=true,      
    colorlinks=true,       
    linkcolor=blue,          
    citecolor=ForestGreen,        
    filecolor=magenta,      
    urlcolor=red,           
    breaklinks=false,
}

\usepackage{cleveref}
\usepackage{thmtools}
\usepackage{thm-restate}

\numberwithin{equation}{section}

\date{}

\title{Sharp Rates in Dependent Learning Theory:\\Avoiding Sample Size Deflation for the Square Loss}

\newcommand{\e}{\varepsilon}

\newcommand{\E}{\mathbf{E}}

\newcommand{\V}{\mathbf{V}}

\DeclareMathOperator*{\argmin}{argmin}

\newcommand{\T}{\mathsf{T}}

\newcommand{\Z}{\mathbb{Z}}

\newcommand{\scrH}{\mathscr{H}}
\newcommand{\scrF}{\mathscr{F}}
\newcommand{\scrG}{\mathscr{G}}

\newcommand{\sfX}{\mathsf{X}}
\newcommand{\sfY}{\mathsf{Y}}

\newcommand{\sfP}{\mathsf{P}}
\newcommand{\sfQ}{\mathsf{Q}}
\newcommand{\sfM}{\mathsf{M}}

\newcommand{\R}{\mathbb{R}}
\newcommand{\N}{\mathbb{N}}
\renewcommand{\Pr}{\mathbf{P}}

\newcommand{\iid}{iid}

\newtheorem{definition}{Definition}[section] 
\newtheorem{theorem}{Theorem}[section] 
\newtheorem*{theorem*}{Theorem} 
\newtheorem{proposition}{Proposition}[section] 

\newtheorem{lemma}{Lemma}[section]

\begin{document}

\maketitle

\begin{abstract} 
In this work, we study statistical learning with dependent ($\beta$-mixing) data and square loss in a hypothesis class $\scrF\subset L_{\Psi_p}$ where $\Psi_p$ is the norm $\|f\|_{\Psi_p} \triangleq  \sup_{m\geq 1}  m^{-1/p} \|f\|_{L^m} $ for some $p\in [2,\infty]$. Our inquiry is motivated by the search for a sharp noise interaction term, or variance proxy, in learning with dependent data. Absent any realizability assumption, typical non-asymptotic results exhibit variance proxies that are deflated \emph{multiplicatively} by the mixing time of the underlying covariates process. We show that whenever the topologies of $L^2$ and $\Psi_p$ are comparable on our hypothesis class $\scrF$---that is, $\scrF$ is a weakly sub-Gaussian class: $\|f\|_{\Psi_p} \lesssim \|f\|_{L^2}^\eta$ for some $\eta\in (0,1]$---the empirical risk minimizer achieves a rate that only depends on the complexity of the class and  second order statistics in its leading term. Our result holds whether the problem is realizable or not and we refer to this as a \emph{near mixing-free rate}, since direct dependence on mixing is relegated to an additive higher order term. We arrive at our result by combining the above notion of a weakly sub-Gaussian class with mixed tail generic chaining. This combination allows us to compute sharp, instance-optimal rates for a wide range of problems. Examples that satisfy our framework  include sub-Gaussian linear regression, more general smoothly parameterized function classes, finite hypothesis classes, and bounded smoothness classes.


\end{abstract}

\newpage
\tableofcontents
\newpage

\section{Introduction}
\label{sec:intro}

While a significant portion the data used in modern learning algorithms exhibits temporal dependencies, we still lack a sharp theory of supervised learning from dependent data. Examples exhibiting such dependencies are far ranging and abundant, and include
forecasting applications and data from controls/robotics systems.
%
Over the last several decades, an order-wise rather sharp theory of learning with \emph{independent} data has emerged. An entirely incomplete list of these advances includes the introduction of local Rademacher compleixities by \citet{bartlett2005local}, sharp rates in misspecified linear regression by \citet{hsu2012random},  and culminates in the learning without concentration framework by \citet{mendelson2014learning}, which enables an instance-optimal understanding of many standard learning problems through a \emph{critical radius} that is sensitive to both the noise scale and the (local) geometry of the hypothesis class.

In principle, one expects these results to be carried over to the dependent ($\beta$-mixing) setting through \emph{blocking}~\citep{bernstein1927extension, yu1994mixing}.\footnote{See \Cref{sec:blocking} for a description of this technique.} At a high level, the blocking technique involves splitting the original data (of length $n\in \N$) into consecutive blocks, each of length $k\in \N$, with the length chosen such that the starting points of each block are approximately independent. Indeed, several prior works pursue this route \citep{mohri2008rademachermixing, kuznetsov2017generalization, roy2021dependent}. However, the drawback with this approach is that it typically deflates the original sample size by the block length factor $k$. If such a deflation were to appear in the final rate of convergence, this would clearly constitute worst-case behavior; it corresponds to every data point being revealed repeatedly, $k$ times and with perfect dependence, within a sequence of $n$ observations.

In the context of the square loss function, the typical approach to sidestep this sample size deflation relies on the ``noise'' (residual term) forming a martingale difference sequence. This approach has been carried out for parametric inference in (generalized) linear dynamical systems by \citet{simchowitz2018learning} and \citet{kowshik2021near} and also for more general hypothesis classes and supervised learning with square loss by \citet{ziemann2022learning}.  For the square loss function the martingale approach requires that the problem is strongly realizable: the best predictor in the hypothesis class should coincide with the regression function (conditional expectation of targets given past inputs). Put differently, one requires that the hypothesis class is rich enough such that conditional expectation function (of the targets and given past inputs) can be realized by it. 

In this paper, we instead show how the blocking approach can be salvaged for a wide range of hypotheses classes and the square loss function. In contrast to the just-mentioned references, our analysis does not require a realizability assumption. Instead, we show how to extend the analysis of \citet{ziemann2023noise} for linear regression to more general hypothesis classes. At a high level, this analysis involves combining the above-mentioned blocking technique with Bernstein's inequality. To motivate this approach, let us consider what happens in Bernstein's inequality when we are given $V_{1:n}$ $b$-bounded random variables that are $k$-wise independent, where $k$ divides $n$, and with identical marginals (for simplicity alone).\footnote{We say that a sequence $Z_{1:n}$ is $k$-wise independent if each of the blocks $Z_{jk+1:(j+1)k}$ ($j=0,1,\dots,n/k-1$) are independent of each other.} By applying Bernstein's inequality to the $bk$-bounded variables $\bar V_{i:n/k}, \bar V_i \triangleq  \sum_{j=ik-k+1}^{ik} V_j$ we find that with probability at least $1-\delta$:
\begin{equation}\label{eq:bernsteininequalityblocked}
    \frac{1}{n}\sum_{i=1}^n V_i \leq 2\sqrt{\frac{ k^{-1} \E  (\bar V_1)^2 \ln (1/\delta)} {n}}+\frac{4bk \log(1/\delta)}{3n}.
\end{equation}
If the data instead were completely independent, then in the small and moderate deviations regime $\delta \gtrsim \exp(-n\E V_1^2/b^2 k)$, \eqref{eq:bernsteininequalityblocked} is just as sharp as directly applying Bernstein's inequality to the independent sum.  In this regime for this problem, nothing is lost by blocking, even if the data happens to be iid and we use the blocked version of Bernstein's inequality. By contrast, if one were to carry out the same computation using Hoeffding's inequality (for bounded random variables) instead of Bernstein's, we would incur an irreducible factor $k$ in the leading term in all regimes---even if the dependent bound is instantiated for independent variables. This suggests that the variance interacts much more gracefully with blocking arguments than higher order moments. 

The difficulty in combining blocking with Bernstein's inequality lies in making Bernstein's inequality uniform across the correct portion of the hypothesis class $\scrF$. Namely, in statistical learning it typically does not suffice to control sums of a single sequence of random variables $V_{1:n}$ but rather we need to uniformly control sums of an indexed family $\{V_{1:n}(f) : f\in \scrF \}$. To obtain fast rates, this uniform control needs to combined with a localization argument, so that one does ``pay'' for hypotheses too far away from the ground truth but only those within a certain critical radius. Na{\"{i}}vely union-bounding (or chaining) over such a family unfortunately again reintroduces a sample-size deflation by the block-length factor $k$. This happens because the variance term in \eqref{eq:bernsteininequalityblocked} starts to balance the boundedness term at the above-mentioned critical radius without further assumption.
\citet{ziemann2023noise} show how to overcome the issue of uniformity when $\scrF$ is a linear class via the Fuk-Nagaev inequality \citep{einmahl2008characterization}. Unfortunately, this inequality cannot be applied beyond the linear setting. Here, we introduce machinery based on a refinement of sub-Gaussian classes \citep{lecue2013learning}, and a refinement of Bernstein's inequality (due to \citet{maurer2021concentration}), that we combine with mixed-tail generic chaining (as introduced by \citet{dirksen2015tail}). Our approach allows us to overcome this issue with blocking and Bernstein's inequality for a surprisingly wide range of function classes, thereby relegating any dependence on mixing to additive higher order terms, instead of the typical multiplicative deflation term.
 
\subsection{Contribution}
\label{sec:contribution}
Let us now make our contribution more precise. We are given stationary $\beta$-mixing data $(X,Y)_{1:n}$ where the $X_i $ (resp.~$Y_i$) assume values in a subset of a normed space denoted $(\sfX, \|\cdot \|_{\sfX})$ (resp.~a Hilbert space $(\sfY,\langle \cdot, \cdot\rangle, \|\cdot\|)$). We assume that $(X,Y)_{1:n}$ is stationary and denote for any $i\in [n]$ the joint distribution of $(X_i,Y_i)$ by $\sfP_{X,Y}$, and the corresponding marginals are denoted $\sfP_X$ and $\sfP_Y$. We study empirical risk minimization over a hypothesis class $\scrF$, containing functions $f : \sfX \to \sfY$, and with the square loss function. In this scenario, we study the performance of the (any) empirical risk minimizer
\begin{equation}\label{eq:ermdef}
    \widehat f \in \argmin_{f \in \scrF}\frac{1}{n} \sum_{i=1}^n \| f(X_i)-Y_i\|^2.
\end{equation}

Our main contribution is to characterize the rate of convergence of \eqref{eq:ermdef} to the best possible predictor $f_\star$ in the class $\scrF$ defined as:
\begin{equation}
    f_\star \in \argmin_{f \in \scrF}  \E \| f(X)-Y\|^2, \qquad (X,Y)\sim \sfP_{X,Y}.
\end{equation}
Let us also denote by $\scrF_\star$ the star-hull of $\scrF$ around $f_\star$.
That is, $\scrF_\star \triangleq \{\rho (f-f_\star) : f\in \scrF, \rho \in [0,1] \}$, which, for a convex class $\scrF$ coincides with the shifted class $\scrF-\{f_\star\}$. We further equip $\scrF_\star$ with the $L^2$-norm: $\|f\|_{L^2}^2 \triangleq  \E \| f(X)\|^2, f \in \scrF_\star, X\sim \sfP_X.$ Let us also define the ``noise'' $W_{1:n}$ by $W_i\triangleq Y_i-f_\star(X_i), i \in [n]$.  We focus on the case when $\scrF$ is either (1) convex or (2) realizable (i.e., $\E[W_i| X_i]=0$ for $i \in [n]$). Note that
this restriction is due to a known shortcoming of ERM which holds even in iid settings,
and can be removed by modifying the estimator itself; we will discuss
this issue in more detail shortly.


As is typical in the learning theory literature, we characterize the rate of convergence of \eqref{eq:ermdef} through a fixed point, or critical radius. This critical radius takes the form as a solution to:
\begin{equation}\label{eq:informalcritrad}
    r_\star \asymp   \sup_{g\in \scrF_\star \cap r_\star S_{L^2}}\V\left( \frac{1}{\sqrt{n}}\sum_{i=1}^n  \left\langle W_i, \frac{g(X_i)}{\|g\|_{L^2}} \right\rangle \right)   
    \times 
    \frac{\textnormal{complexity}(\scrF_\star \cap r_\star S_{L^2}) }{r_\star \sqrt{n}},
\end{equation}
where for $r\in \R, r>0$, $rS_{L^2}$ is the unit sphere of radius $r$ in $L^2$ (the space of square integrable functions, and with the corresponding unit ball denoted $rB_{L^2})$ and $\V(\cdot)$ denotes the variance operator. This critical radius is akin to the one in \citet{bartlett2005local}, but also resembles the noise interaction term of \citet[introduced following Equation 2.2]{mendelson2014learning} in that our radius depends on the \emph{weak variance}, $ \sup_{g\in \scrF_\star \cap r_\star S_{L^2}} \V\left( \frac{1}{\sqrt{n}}\sum_{i=1}^n  \left\langle W_i, \frac{g(X_i)}{\|g\|_{L^2}} \right\rangle \right) $.\footnote{The terminology weak variance comes from the empirical processes literature in that the supremum in \Cref{def:weakvariancedef} is on the outside of the expectation.} To aid in the interpretation of $r_\star$, we will instantiate our main result, \Cref{thm:themainthm}, for parametric classes and show that this radius exhibits the desired ``dimension counting'' scaled with noise-to-signal behavior, see \Cref{cor:params} and \Cref{cor:linreg}. Moreover, the weak variance term takes into account how targets $Y_{1:n}$ interact with the function class $\scrF$ through $W_{1:n}$, locally at radius $r_\star$ near the minimizer $f_\star$, via a second-order statistic. In particular, this variance term is always sharper than the corresponding \iid\ variance term deflated by a factor of the mixing-time (or block-length).

With these preliminaries in place we are ready to state an informal version of our main result.
\paragraph{Informal version of \Cref{thm:themainthm}.} \emph{Given data that mixes sufficiently fast, for a wide range of (1) convex or (2) realizable hypothesis classes, any empirical risk minimizer $\widehat f$ over such a class $\scrF$ converges at least as fast a rate characterized by the critical radius $r_\star$ given by the solution to \eqref{eq:informalcritrad} depending on the variance of the noise-class interaction and local scale of the class $\scrF$. That is with probability $1-\delta$:}
\begin{multline}
\label{eq:mainthminformal}
    \| \widehat f-f_\star\|_{L^2}^2 \lesssim r_\star^2 + \frac{(\textnormal{weak variance}) \times \log(1/\delta)}{n}\\+\textnormal{terms of higher order}(r_\star,n^{-1},\mathrm{mixing},\log(1/\delta)).
\end{multline}
\emph{Moreover, for $d$-dimensional parametric classes the leading term is $  (\textnormal{weak variance}) \times \frac{ d+\log(1/\delta)}{n}$.}

The crux of this result is that past a burn-in, the ERM excess risk does not directly depend on mixing times, but only on the relevant second order statistics. Put differently, the effect of slow mixing has been relegated to a small \textit{additive} term with higher order dependence on $1/n$. Indeed, both $r_\star$ and the variance term in \eqref{eq:mainthminformal} do not directly depend on slow mixing (i.e., are not deflated by the block-length $k$) but only on relevant second order statistics. Slow mixing only affects higher order additive terms that can be pushed into the burn-in.

The qualifier ``wide range'' above refers to the requirement that the class $\scrF$ satisfies a certain topological condition.  Recall that for a random variable $Z$ the $\Psi_p$-norm is the norm $\|Z\|_{\Psi_p} = \sup_{m\geq 1} m^{-1/p}\|Z\|_{L^m} $. We will ask that for some $\eta\in (0,1]$ and $L>0$, every $f \in \scrF_\star$ satisfies the inequality  $\|f\|_{\Psi_p} \leq L \|f\|_{L^2}^\eta$. We will say that such classes are \emph{weakly sub-Gaussian} and will verify that such an inequality indeed holds true for a range of examples in \Cref{sec:examples}:
\begin{itemize}
    \item bounded smoothness classes, see \Cref{prop:smoothnessclass};
    \item parametric classes that are Lipschitz in their parameterization, see \Cref{prop:sharpnessprop};
    \item sub-Gaussian linear regression, see \Cref{prop:linreg};
    \item finite hypothesis classes, see \Cref{prop:finhyp}.
\end{itemize}

Finally, the requirement that $\scrF$ be either (1) convex or (2) realizable can easily be removed with a few modifications if one replaces the empirical risk minimizer by the star estimator of \citet{audibert2007progressive}. 
In this case (but with the $L^2$-error replaced with the no-longer directly comparable excess risk functional) the geometric inequality by \citet[Lemma 1]{liang2015learning} takes a similar role to the basic inequality we use below. The necessity of imposing (1) or (2) is due to a known shortcoming of empirical risk minimization outside of convex (or realizable) classes, and not an issue directly related to dependent data \citep[see e.g. the discussion in][]{Mendelson2019unrestricted}.

\subsection{Proof Outline}

From a more technical standpoint, our contribution is a novel analysis of two empirical processes that arise in (but are not restricted to) empirical risk minimization, and which are sharp even for dependent data. Following the language of \citet{mendelson2014learning}, we refer to these as the quadratic and multiplier empirical processes.  The first of these, the quadratic process, controls a one-sided discrepancy between the empirical and population $L^2$-norms:
\begin{equation}
\label{eq:quadraticdefined}
   Q_n(f) \triangleq   \frac{ 1}{n}\sum_{i=1}^n \E \| f(X_i)-f_\star(X_i)\|^2 
   -\frac{(1+\e)}{n}\sum_{i=1}^n \| f(X_i)-f_\star(X_i)\|^2,
\end{equation}
for some $\e \in (0,1)$. Under our assumptions, we will show that the process $Q_n(f)$ is eventually nonpositive uniformly for all  sufficiently large $f$, implying that the empirical $L^2$-norm dominates the population $L^2$-norm. 

Now for $r\in \R_+$ and conditionally on the event $\{ Q_n(f) \leq 0, \: \forall f \in\scrF_\star\setminus rB_{L^2}  \}$, using optimality of $\widehat f$ to \eqref{eq:ermdef} we also have the following deterministic (basic) inequality:
\begin{equation}\label{eq:basicineq}
 \| \widehat f - f_\star\|_{L^2} \leq r
 +\frac{1+\e}{r n}\sum_{i=1}^n 2(1-\E') \Bigg[
 \left\langle W_i, \frac{r\left[\widehat f(X_i)-f_\star(X_i)\right]}{\|\widehat f-f_\star\|_{L^2}} \right\rangle \Bigg],
\end{equation}
where $\E'$ denotes expectation with respect to a fresh copy of randomness independent of $\widehat f$.

Hence, we also need to control the multiplier process:
\begin{equation}
\label{eq:multiplierdefined}
     M_n(f) \triangleq\frac{1+\e}{n}\sum_{i=1}^n2(1-\E') \langle W_i, f(X_i) \rangle.  
\end{equation}
It is the uniform control of $M_n(f)$ over the class $\scrF_\star$ intersected with the radius $r$ ball $r S_{L^2}$ balanced with the first term of \eqref{eq:basicineq} that gives rise to the critical radius \eqref{eq:informalcritrad}. This argument is formalized in \Cref{lem:basicineq}. Just as in \citet{mendelson2014learning}, it is the multiplier process \eqref{eq:multiplierdefined} that yields the dominant contribution to the error \eqref{eq:mainthminformal} (after a burn-in). This is important as it allows us to use blocking to control \eqref{eq:quadraticdefined} without affecting the leading term of the final rate.

We reiterate that our analysis of the above two empirical processes (\eqref{eq:quadraticdefined} and \eqref{eq:multiplierdefined}) rests crucially on the assumption that $\scrF_\star$ is a weakly sub-Gaussian class. Let us also point out that we first make a simplifying assumption, namely that our model is $k$-wise independent. We later port all results to the $\beta$-mixing setting by blocking, cf. \Cref{sec:betamixing}. A sketch of the analysis of $M_n(f)$---found in \Cref{sec:multiplier} with proofs relegated to \Cref{sec:multiplierapp}---now goes as follows:
\begin{itemize}
    \item We invoke a refinement of Bernstein's inequality (\Cref{lem:phibernsteinmgf}) to gain pointwise control of $M_n(f)$. The benefit of this over the standard version is that we do not require boundedness, but rather finite $\Psi_p$-norm suffices. Unless $p=\infty$ (boundedness), the price we pay for this is that the variance proxy is degraded to a moment of order $2q,q>1$ instead of order $2$.
    \item We make this refinement of Bernstein's inequality uniform  over the class $\scrF_\star$ intersected with the radius $r$ ball $r S_{L^2}$ by invoking mixed-tail generic chaining \citep{dirksen2015tail}.  This splits the tail into an $L^{2q}$-component and a $\Psi_p$-component.
    \item Our assumption that $\scrF_\star$ is a weakly sub-Gaussian class now comes into play by ensuring that, past a burn-in, the $\Psi_p$-component of the mixed tail is of lesser magnitude than the $L^{2q}$-part of the tail. Just as in our introductory example with Bernstein's inequality \eqref{eq:bernsteininequalityblocked}, any dependence on mixing is relegated to this smaller $\Psi_p$-component (which now assumes the role of boundedness).
    \item Combining these steps with \eqref{eq:basicineq}  yields control of the multiplier process and is summarized in \Cref{thm:multiplierthm}.
\end{itemize}
The analysis of $Q_n(f)$ is relatively standard and amounts to showing that the norm-bound $\|f\|_{\Psi_p} \leq L \|f\|_{L^2}^\eta$ is sufficient to modify a standard truncation argument \citep[see e.g.][Theorem 14.12]{wainwright2019high}. We then proceed to control the remainder of said truncation argument completely analogously to our above approach for $M_n(f)$. We detail these arguments in \Cref{sec:loweruniform} and prove them in \Cref{sec:quadapp}.  Finally, we combine our control of the multiplier and quadratic processes (\Cref{thm:multiplierthm} and \Cref{thm:loweruniform}) with blocking to arrive at our main result, \Cref{thm:themainthm}. 


\subsection{Further Preliminaries}

\paragraph{Notation.}
Expectation (resp.\ probability) with respect to all the randomness of the underlying probability space is denoted by $\E$ (resp.\ $\Pr$). For $q\in [1,\infty)$ the $2q$-variance of a random variable $Z$ is defined as $\V_{2q}(Z) \triangleq  (\E(Z-\E Z)^{2q})^{1/q} $ with $\V_2=\V$ being the standard variance. For $p\in [1,\infty)$, we also introduce the $\Psi_p$-norm  $\|Z\|_{\Psi_p} \triangleq \sup_{m\geq 1}m^{-1/p} \|Z\|_{L^m} $ and also set $\|Z\|_{\Psi_\infty} \triangleq  \|Z\|_{L^\infty}.$ Two extended real numbers $q,q' \in [1,\infty]$ are said to be Hölder conjugates if $1/q + 1/q' = 1$, where, as we do throughout, $1/\infty$ is interpreted as $0$.  For two probability measures $\sfP$ and $\sfQ$ defined on the same probability space, their total variation is denoted $\|\sfP-\sfQ\|_{\mathsf{TV}}$. Maxima (resp.\ minima) of two numbers $a,b\in \R$ are denoted by $a\vee b =\max(a,b)$ (resp.\ $a\wedge b = \min(a,b)$). For an integer $n\in \N$, we also define the shorthand $[n] \triangleq\{1,\dots,n\}$. For a symmetric positive semidefinite  matrix $M$, $\lambda_{\mathrm{min}}(M)$ denotes its smallest nonzero eigenvalue.

\paragraph{Talagrand's functionals.}
The $\textnormal{complexity}(\scrF_\star \cap r_\star S_{L^2})$ term in \eqref{eq:informalcritrad} is made precise through Talagrand's $\gamma_{\alpha}${-}functional (with $\alpha=2$ being the dominant term in our result). Let be $(\scrH,d)$ a metric space.  We denote the diameter of $\scrH$ with respect to $d$ by
\[
\Delta_d(\scrH) \triangleq \sup_{h,h'\in \scrH} d(h,h').
\]
A sequence $ \mathcal{H}=(H_m)_{m\in \Z_+}$ of subsets of $\scrH$ is called admissible if $|H_0| = 1$ and $|H_m| \leq 2^{2^m}$ for all $m \geq 1$. For $\alpha \in (0,\infty)$, the $\gamma_\alpha$-functional of $(\scrH,d)$ is defined by
\begin{equation}\label{eq:gammafuncdef}
\gamma_\alpha(\scrH,d) \triangleq \inf_{\mathcal{H}} \sup_{h\in \scrH} \sum_{m=0}^{\infty} 2^{m/\alpha} d(h,H_n),   
\end{equation}
where the infimum is taken over all admissible sequences (we write \(d(h, H) = \inf_{s\in H} d(h, s)\) whenever $H$ is a set). For $\eta \in (0,1)$, we slightly abuse notation and write $\gamma_\alpha(\scrH,d^\eta)$ for $d$ replaced with $d^\eta$ in \eqref{eq:gammafuncdef} (while being mindful of that fact that $d^\eta$ is not a metric in general). Finally, since entropy integrals upper-bound $\gamma_\alpha$-functionals, it will also be useful to introduce 
the covering number $\mathcal{N}_{L^2}(\mathcal{\scrH},s)$,
which denotes the minimal number of $L^2$-balls of radius $s$ required to cover $\scrH$.

\section{$\Psi_p$-Norms, Bernstein's Inequality and Empirical Processes}

In this section we establish a few preliminary technical lemmas that will be useful for controlling the multiplier  and quadratic  processes (\eqref{eq:multiplierdefined} and \eqref{eq:quadraticdefined}). We begin with a  version of Bernstein's inequality that controls the Laplace transform of $Z$ in terms of its $L^{2q}$-norm ($q\geq 1$) and some $\Psi_p$-norm. The  lemma comes from \cite{maurer2021concentration}.

\begin{restatable}[$\Psi_p$-Bernstein MGF Bound]{lemma}{phibernsteinmgf}\label{lem:phibernsteinmgf}
 Fix a random variable $Z$ and $p\in[1,\infty]$ such that $\E Z \leq 0$ and $\|Z\|_{\Psi_p}<\infty$. Let $q$ and $q'$ be Hölder conjugates and suppose that $\lambda \in [0,1/(q' e)^{1/p} \| Z\|_{\Psi_p}]$. We have that:
    \begin{equation}\label{eq_lem:phibernsteinmgf}
        \E \exp \left(\lambda Z \right) \leq  \exp \left( \frac{\frac{\lambda^2}{2}\left(\E ( Z)^{2q}\right)^{1/q}}{1-\lambda (q' e)^{1/p} \| Z\|_{\Psi_p}}\right).
    \end{equation} 
\end{restatable}

Our intention is to use \Cref{lem:phibernsteinmgf} to afford us---pointwise in $g$---control of the multiplier process introduced in \eqref{eq:multiplierdefined}. Indeed, notice that in the regime $\lambda \in (0, (2(q'e)^{1/p} \|Z\|_{\Psi_p})^{-1} ]$ the dominant term in \eqref{eq_lem:phibernsteinmgf} is $2q$-variance of $Z$. Since \eqref{eq:basicineq} can be localized to a ball of radius $r$ in $L^2$ it suffices that the $L^2$-norm provides some weak control of the $\Psi_p$-norm for any constant choice of $\lambda$ to be admissible once the localization radius $r$ is chosen small enough. This in turn motivates the following definition.

\begin{definition}[Weakly sub-Gaussian Class]
\label{def:weaksgc}
    Fix $\eta \in (0,1]$ and $L \in [1,\infty)$. We say that a class $\scrG$ is $(L,\eta)$-$\Psi_p$  if for every $g\in \scrG$ we have that: 
    \begin{equation}\label{df_eq:phiclass}
       \left\|  g \right\|_{\Psi_p} \leq L \left\| g\right\|_{L^2}^{\eta}.
    \end{equation}
\end{definition}


If  \eqref{df_eq:phiclass} holds for $\scrG$ with $\eta\in(0,1)$ and some $L$ we will call $\scrG$ a weak $\Psi_p$-class. If \eqref{df_eq:phiclass} instead holds for $\eta=1$ it is simply a $\Psi_p$-class. This generalizes the notion of a sub-Gaussian class from \cite{lecue2013learning}, which corresponds to $\eta =1$ and $p=2$. Let us further point out that by homogeneity, if $\eta\in (0,1)$ in \eqref{df_eq:phiclass}, then one should expect  $L$ to depend polynomially on some other norm (or homogenous functional) of $g$. Indeed, by the Gagliardo-Nirenberg interpolation inequality, the above relaxation ($\eta <1$)  covers smoothness classes (\Cref{prop:smoothnessclass}), whereas the strict sub-Gaussian class assumption ($\eta=1$) of \cite{lecue2013learning} is difficult to verify beyond linear functionals.

As we have pointed out above, our intention is to apply \Cref{lem:phibernsteinmgf} pointwise to the multiplier  process \eqref{eq:multiplierdefined}. However, this yields a different variance term for each index point of the empirical process. The solution to this is simply to define a uniform variance term, as is done below.

\begin{definition}[Noise Level]\label{def:weakvariancedef}
The $2q$-noise-class-interaction between $\scrF$, the model  $\sfP_{(X,Y)_{1:n}}$, and the shifted target $W_{1:n}=(Y-f_\star(X))_{1:n}$ at resolution $\scrG$ is given by
\begin{equation}
    \V_{2q}\left(\scrF,\scrG, \sfP_{(X,Y)_{1:n}}\right)
    \triangleq  \sup_{g\in \scrG}\V_{2q}\left( \frac{1}{\sqrt{n}}\sum_{i=1}^n  \left\langle W_i, \frac{g(X_i)}{\|g\|_{L^2}} \right\rangle \right).
\end{equation}
\end{definition}

We stress that, even though \Cref{def:weakvariancedef} measures noise uniformly over a function class, it does not generally grow with the complexity of the class. For instance, under the additional hypotheses that $\sfP_{(X,Y)_{1:n}}$ is drawn \iid\ and that $W_i$ is independent of $X_i$ for $i\in[n]$, it is easy to see that $  \V_{2}\left(\scrF,\scrG, \sfP_{(X,Y)_{1:n}}\right)=\V_2(W)$ for every such well-specified class $\scrG$. Rather, \Cref{def:weakvariancedef} is a measure of how well the targets $Y_{1:n}$ align with a given class $\scrG$.


\subsection{The Multiplier Process}
\label{sec:multiplier}

We will not directly control the multiplier process for $\beta$-mixing variables. Instead we first suppose that 
the model $\sfP_{(X,Y)_{1:n}}$ is $k$-wise independent (where $k$ divides $n$). We then port these results to the $\beta$-mixing setting by blocking (see~\Cref{sec:blocking}). We use the following 
shorthand notation regarding $ \V_{2q}\left(\scrF,\scrG, \sfP_{(X,Y)_{1:k}}\right)$:  we take the class $\scrF$ and the probability model $\sfP_{(X,Y)_{1:k}}$  as fixed and thus omit the dependence on $\scrF$ (via $f_\star$) and $\sfP_{(X,Y)_{1:k}}$ and write $\V_{2q}\left(\scrG\right)=\V_{2q}\left(\scrF,\scrG, \sfP_{(X,Y)_{1:k}}\right) $. With these remarks in place, we now turn to establishing pointwise control of \eqref{eq:multiplierdefined} using \Cref{lem:phibernsteinmgf}.

\begin{restatable}[Pointwise Control]{lemma}{bernsteinemppt}\label{lem:bernsteinemppt}

  Fix two Hölder conjugates $q$ and $q'$.  Suppose that the model $\sfP_{(X,Y)_{1:n}}$ is stationary and $k$-wise independent where $k$ divides $n$. For every $g,g'\in L_{\Psi_p}$ and $u\in (0,\infty)$ we have that:
        \begin{multline}\label{eq_lem:bernsteinemppt}
    \Pr \Bigg( \sum_{i=1}^n (1-\E) \langle W_i, g(X_i)-g'(X_i)\rangle \\
    > \sqrt{4 n \|g-g'\|_{L^2}^2  \V_{2q}\left(\{g\}-\{g'\}  \right) u}\\
    + 
    4 (q' e)^{2/p}  k \|  W \|_{\Psi_{p}} \| g-g'\|_{\Psi_p} u \Bigg)
    \leq 2e^{-u}.
\end{multline}
\end{restatable}

In the main development we will instantiate \Cref{lem:bernsteinemppt} with $r=\|g-g'\|_{L^2}^2$ decaying to $0$ (which should be thought of as a fixed point upper-bounding the rate of convergence of ERM)  at a polynomial rate in $n$. If furthermore $\scrG$ is $(L,\eta)$-$\Psi_p$, then the second term (linear in $u$) of \eqref{eq_lem:bernsteinemppt} can be rendered negligible at every scale $r$, which allows us to invoke mixed-tail generic chaining \citep{dirksen2015tail} to show that the weak variance $\V_{2q}\left(\scrF_\star \cap r S_{L^2}\right)$ dominates the noise level in the small-to-moderate deviations regime.

Put differently,  at the scale of localization considered here, the noise level of the empirical process is almost entirely dictated by the weak variance $\V_{2q}(\scrF_\star \cap r S_{L^2})$. 
Now, since $q'$ is the Hölder conjugate of $q$ this further implies that we may choose $q=1+o(1)$ so that we might expect $\V_{2q}\left(\scrF_\star \cap r S_{L^2}\right) = \V\left(\scrF_\star \cap r S_{L^2}\right)+o(1)$. Moreover if $p=\infty$ this always the case and we may choose $q=1$.  In principle no better variance proxy is possible, since already for a single function $g$ as $n\to\infty$,
by the central limit theorem under mild ergodicity assumptions on $\sfP_{(X,Y)_{1:\infty}}$ \citep[e.g.\ for the Markovian situation cf.][Theorem 17.3.6]{meyn1993markov}: 
\begin{equation}\label{eq:cltheuristic}
    \frac{1}{\sqrt{n}}\sum_{i=1}^n (1-\E) \left\langle W_i, \frac{g(X_i)}{\|g\|_{L^2}}\right\rangle
    \rightsquigarrow N(0,\V(\scrF_\star, \{g\},\sfP_{(X,Y)_{1:\infty}})),
\end{equation} 
where the variance term on the right is:
$$\V(\scrF_\star, \{g\},\sfP_{(X,Y)_{1:\infty}}) \triangleq \lim_{n\to\infty}\V(\scrF_\star, \{g\},\sfP_{(X,Y)_{1:n}}).$$
%
Now, since $r=o(1)$ in all practical situations one expects $\V(\scrF_\star \cap r S_{L^2}) \approx \V(\{f_\star \})$ as long as the map $f \mapsto \V (f/\|f\|_{L^2})$ is sufficiently regular near $f_\star$.

We arrive at our main result for the multiplier process by making uniform the pointwise control afforded to use by \Cref{lem:bernsteinemppt} via an instantiation of mixed-tail generic chaining \citep{dirksen2015tail} (for ease of reference, 
we restate a corollary of his result as \Cref{thm:dirksens} in the appendix).
This yields the following result.
\begin{restatable}{theorem}{multiplieruniform}
    
        \label{thm:multiplierthm}
      Fix a failure probability $\delta \in (0,1)$, a positive scalar $r\in (0,\infty)$, two Hölder conjugates $q$ and $q'$, and a class $\scrF$. Suppose that $\scrF_\star -\scrF_\star$ is $(L,\eta)$-$\Psi_p$. Suppose further that the model $\sfP_{(X,Y)_{1:n}}$ is stationary and $k$-wise independent where $k$ divides $n$. There exist universal positive constants $c_1,c_2$ such that for any $r\in (0,1]$ we have that with probability at least $1-\delta$:
    \begin{multline}\label{eq:multiplierthmeq}
         \sup_{f\in \scrF_{\star}\cap rS_{L^2}}\frac{1}{rn}\sum_{i=1}^n (1-\E) \langle W_i, f\rangle  \\
         \leq  c_2 \sqrt{ \V_{2q}\left(\scrF_{\star}\cap rS_{L^2}\right) }  \Bigg(\frac{1}{r\sqrt{n}}
         \gamma_2(\scrF_{\star}\cap rS_{L^2},d_{L^2}) + \sqrt{\frac{\log(1/\delta)}{n}} \Bigg) \\
         +c_1(q' e)^{2/p}  L  k \|W\|_{\Psi_p}
          \times
         \left(\frac{1}{rn}\gamma_\eta(\scrF_{\star}\cap rS_{L^2},d_{L^2}) +\frac{r^{\eta-1}}{n} \log(1/\delta)\right).
    \end{multline}

\end{restatable}

In the sequel, we will see that the first term on the right of \eqref{eq:multiplierthmeq} is typically dominant.

\subsection{The Quadratic Process}
\label{sec:loweruniform}

A slight modification of the argument leading to \Cref{thm:loweruniform} combined with a truncation argument detailed in \Cref{lem:trunclem} also yields control of the quadratic process.

\begin{restatable}[Lower Uniform Law]{theorem}{loweruniformlaw}\label{thm:loweruniform}
Fix a failure probability $\delta \in (0,1)$, a tolerance $\e>0$, a localization radius $r\in (0,1]$, and two Hölder conjugates $q$ and $q'$. Suppose that $\scrF_\star-\scrF_\star$ is $(L,\eta)$-$\Psi_p$. Suppose further that the model $\sfP_{(X,Y)_{1:n}}$ is stationary and $k$-wise independent where $k$ divides $n$. There exist a universal positive constant $c$ such that uniformly for all $f\in \scrF_\star \setminus \{rB_{L^2}\}$ we have that with probability at least $1-\delta$: 

\begin{multline}
\frac{1}{n}\sum_{i=1}^n \|f(X_i)\|^2 \geq  r^2(1-\e^2)
\\
- c \Bigg \{ n^{-1/2}\sqrt{k}    L^{1+3/4}r^\eta \left(  \log \left(\frac{4^{2/p}L }{\e r} \right)  \right)^{1/p}  
\left( \gamma_{\frac{2+6\eta}{4}}(\scrF_{\star}\cap rS_{L^2},d_{L^2})+ r^{\frac{1+3\eta}{4}} \sqrt{\log (1/\delta)}\right)
     \\
     +
      n^{-1}  (q' )^{1/p} k          r^\eta \left(  \log \left(\frac{4^{2/p}L }{\e r} \right)  \right)^{1/p}                    L^2 
      \left( \gamma_\eta(\scrF_{\star}\cap rS_{L^2}, d_{L^2})
     +r^\eta \log(1/\delta) \right) \Bigg\}.
\end{multline}
\end{restatable}

\subsection{$\beta$-Mixing Processes}
\label{sec:betamixing}

We extend the empirical process results of the preceding two sections to $\beta$-mixing processes in \Cref{sec:empprocessmix}. We do so by a simple blocking argument that we review in \Cref{sec:blocking}, and for which we have already set the stage by establishing our results for $k$-wise independent processes. Here, we state the definition of dependence we rely on in the sequel.

\begin{definition}\label{def:beta_mixing_coeffs}
Let $Z_{1:n}$ be a stochastic process. The $\beta$-mixing coefficients of $Z$, denoted $\beta_Z(i)$, are for $i \in [n]$:
\begin{equation}\label{eq:betacoeffs}
    \beta_{Z}(i) \triangleq  \sup_{t\in[n]: t+i\leq n} \E \| \sfP_{Z_{i+t}}( \cdot \mid Z_{1:t} ) - \sfP_{Z_{i+t}}(\cdot) \|_{\mathsf{TV}}.
\end{equation}
\end{definition}

\section{The Main Result}

Before we state our main result, we will need to establish one more preliminary matter. Let us define the burn-in times $n_{\mathsf{quad}},n_{\mathsf{mult}},k_{\mathsf{mix}} $ which together dictate the minimial sample size necessary for our result to be sharp. The first of these, $n_{\mathsf{quad}}$, is required for the population $L^2$ error to be dominated by the empirical $L^2$ error: i.e., the quadratic process $Q_n(f)$ is nonpositive on our class of interest. The second of these, $n_{\mathsf{mult}}$, is required for the multiplier process, $M_n(f)$, to have a dominant variance term (informally---when the CLT-like rate becomes accurate). Finally, $k_{\mathsf{mix}}$ is the minimal block-size it takes for the $\beta$-mixing model $\sfP_{(X,Y)_{1:n}}$ to be well-approximated by a corresponding $k$-wise independent model. These are given as follows below:
\begin{equation}
\label{eq:burnindef}
    \begin{aligned}
        n_{\mathsf{quad}}(r)&= \inf \Bigg\{ n\in \N \Bigg| \Bigg [n^{-1/2}\sqrt{k}    L^{1+3/4}r^\eta
        \left(  \log \left(\frac{4^{2/p+1/2}L }{ r} \right)  \right)^{1/p}  
        \left( \gamma_{\frac{2+6\eta}{4}}(\scrF_{\star}\cap rS_{L^2},d_{L^2})+ r^{\frac{1+3\eta}{4}} \sqrt{\log (1/\delta)}\right)
     \\
     &+
      n^{-1}   L^2 (q' )^{1/p} k          r^\eta \left(  \log \left(\frac{4^{2/p+1/2}L }{ r} \right)  \right)^{1/p}                    
      \left( \gamma_\eta(\scrF_{\star}\cap rS_{L^2}, d_{L^2})
     +r^\eta \log(1/\delta) \right) \Bigg] \leq r^2\Bigg\}, \\
      n_{\mathsf{mult}}(r)&=\inf \Big\{ n \in \N\Big| (q' e)^{2/p}  L  k \|W\|_{\Psi_p}
      \left(\frac{1}{rn}\gamma_\eta(\scrF_{\star}\cap rS_{L^2},d_{L^2}) +\frac{r^{\eta-1}}{n} \log(1/\delta)\right) \leq r\Big\}, \\
        k_{\mathsf{mix}}&= \inf \{k \in [n] | k \beta^{-1}_{X,Y}(k) \geq n \delta^{-1}  \}.
    \end{aligned}
\end{equation}
The first two of these are calculated by requiring the remainder terms in \Cref{prop:multipliermix} and \Cref{thm:loweruniformmix} to be of negligible order. The last term is obtained by requiring that failure term, $\delta$, dominates the mixing term, $\frac{n}{k}\beta_{X,Y}(k)$, in the failure probability of these propositions. At this point, as a practical example, it is worth to point out that if the process $(X,Y)_{1:n}$ is geometrically ergodic---$\beta_{X,Y}(k)\lesssim \exp(-k/\tau_{\mathsf{mix}})$ for some $\tau_{\mathsf{mix}} \in \R_+$---this requirement is satisfied by $k \lesssim (1\vee \tau_{\mathsf{mix}})\log(n/\delta)$. With these burn-in times in place, we are now ready to state
the main result of our paper.
\begin{theorem}
\label{thm:themainthm}
Fix a failure probability $\delta \in (0,1)$, two Hölder conjugates $q$ and $q'$, and a class $\scrF$ that is either (1) convex or (2) realizable. Suppose that $\scrF_\star-\scrF_\star$ is $(L,\eta)$-$\Psi_p$. Suppose further that the model $\sfP_{(X,Y)_{1:n}}$ is stationary and  that $k$ divides $n/2$. There exist universal positive constants $c_1,c_2,c_3$ such that the following holds.  If $r_\star$ solves
\begin{equation}\label{eq_thm:critrad}
     r \geq  c_1\sqrt{    \V_{2q}\left(\scrF_{\star}\cap rS_{L^2} \right) }  \times \frac{1}{r\sqrt{n}} \gamma_2(\scrF_{\star}\cap r S_{L^2},d_{L^2}),
\end{equation}
we have that with probability $1-4\delta$ that:
\begin{equation}\label{eq:themainthmeq}
    \| \widehat f - f_\star\|_{L^2}^2\leq c_2 \left(r_\star^2 +   \V_{2q}\left(\scrF_{\star}\cap r_\star S_{L^2}\right)\frac{\log(1/\delta)}{n}\right)
\end{equation}
as long as  $  n \geq c_3 \max\left\{n_{\mathsf{quad}}(r_\star),n_{\mathsf{mult}}(r_\star)\right\}$ and $k\geq k_{\mathsf{mix}}$ (given in \eqref{eq:burnindef}).
\end{theorem}

\Cref{thm:themainthm} informs us that past a burn-in,
the rate of convergence of empirical risk minimization is dictated by the critical radius $r_\star$ given in \eqref{eq_thm:critrad}. This radius depends on local complexity of the class $\scrF$ measured in $L^2$ distance as per the $\gamma_2$-functional and through the weak variance $\V_{2q}(\scrF_\star \cap r_\star S_{L^2})$. We point out that we may choose $q=1$ if $p=\infty$, so that the variance term in \eqref{eq:themainthmeq} is the actual variance $\V_2$. As indicated at the discussion following \eqref{eq:cltheuristic}, this variance term cannot be improved in general. Otherwise we can typically let $q$ approach $1$ as the sample size becomes sufficiently large. Moreover, unless the class exhibits large nonparametric behavior, the dependency on the complexity is also the best possible even in the \iid\ case \citep{lecue2013learning}.


We now turn to parsing \Cref{thm:themainthm} by specializing it to parametric classes. First, in \Cref{cor:params} we show that for parametric classes the complexity term dictated by the critical radius $r_\star$ in \eqref{eq_thm:critrad} becomes a variance(-proxy)-scaled dimensional factor and that the burn-in requirement \eqref{eq:burnindef} amounts to a polynomial in problem data and $\log(1/\delta)$. Second, we simplify matters further and study bounded and realizable linear regression in \Cref{cor:linreg}. In this case, we will see that the variance term $\V_{2q}\left(\scrF_{\star}\cap r_\star S_{L^2}\right)$ in \eqref{eq:themainthmeq} simply reduces to 2-variance of the noise variable $W$. Moreover, the first two burn-in requirements $n_{\mathsf{quad}}$ and $n_{\mathsf{mult}}$ are in this case satisfied as soon as $n/k \gtrsim d+\log(1/\delta)$. 

\begin{restatable}[Parametric Classes]{corollary}{paramclass}
\label{cor:params}
Fix a failure probability $\delta \in (0,1)$, two Hölder conjugates $q,q'$,  and a class $\scrF$ that is either (1) convex or (2) realizable. Suppose that $\scrF_\star-\scrF_\star$ is $(L,\eta)$-$\Psi_p$. Suppose further that the model $\sfP_{(X,Y)_{1:n}}$ is stationary and that $k$ divides $n/2$.

For every $\eta\in  (1/4,\infty)$, there exists a universal positive constant $c$ and a polynomial function $\phi_{\eta}$ such that the following holds true. Suppose that there exists $ d_{\scrF} \in \R_+ $ such that for $s>0$:
    \begin{equation*}
        \log \mathcal{N}_{L^2}(\mathcal{\scrF_\star},s) \leq d_{\scrF} \log  \left(\frac{1}{s} \right) 
    \end{equation*}
    We have with probability $1-4\delta$ that:
\begin{equation*}
    \| \widehat f - f_\star\|_{L^2}^2\leq c\V_{2q}\left(\scrF_{\star}\cap  \sqrt{\frac{d_\scrF k \|W\|_{L^2}^2}{n}} S_{L^2}\right)
    \left( \frac{d_{\scrF}+  \log(1/\delta)}{n}\right)
\end{equation*}
as long as  $k \beta^{-1}(k) \geq n \delta^{-1}$ and 
\begin{equation*}
    n \geq \phi_{\eta} \Bigg(d_\scrF,k,\|W\|_{\Psi_p},L,q,q',
    \V^{-1}\left(\scrF_{\star}\cap  \sqrt{\frac{d_\scrF k \|W\|_{L^2}^2}{n}}S_{L^2}\right),\log(1/\delta) \Bigg).
\end{equation*} 
\end{restatable}

Consequently, after a polynomial burn-in and up to a universal positive constant, we are able to recover the optimal parametric rate $ n^{-1}\big(d_{\scrF}+  \log(1/\delta)\big)$ scaled by the appropriate noise term. Stated in its most general form, the burn-in term \eqref{eq:burnindef} can be somewhat hard to parse while \Cref{cor:params} only shows that this burn-in is polynomial in problem parameters. The next corollary shows that in the case $\eta=1$ our burn-in coincides with the familiar requirement that the (effective) sample size exceeds the number of degrees of freedom. To simplify matters further we now specialize our result to realizable bounded linear regression. Here, one can think of this burn-in as requiring the empirical covariance matrix of the $X$-process to be invertible with high probability.\footnote{A small caveat to this remark is that the factor $\frac{kB_W^2}{\V(W)}$ in \eqref{eq:linregburnin} arises from the multiplier process: it is the cost of having $\V(W)$ appear in \eqref{eq:linregrate} instead of $  kB_W^2$.}

\begin{restatable}[Realizable Linear Regression]{corollary}{linreg}
\label{cor:linreg}
Fix a failure probability $\delta \in (0,1)$, a covariate bound $B_X \in (0,\infty)$ and a noise bound $B_W\in (0,\infty)$ and let $\sfX=\R^d$ and $\sfY=\R$. Suppose that $k$ divides $n/2$ and that the model $\sfP_{(X,Y)_{1:n}}$ is stationary and  satisfies
$
    Y_i = \langle \beta_\star, X_i \rangle +W_i$ for $i \in [n]$. 
Suppose further that:
\begin{enumerate}
    \item $X_{1:n}$ is bounded $|\langle v, X_i\rangle |\leq B_X, \: \forall i \in [n]$ and $v \in \R^d$ with $\|v\|=1$; and
    \item $W_{1:n}$ is a bounded martingale difference sequence---$\E[W_i | X_{1:i}]=0$ and $|W_i|\leq B_W, \: \forall i \in [n]$.
\end{enumerate}
There exist universal positive constants $c_1$ and $c_2$ such that if
\begin{multline}\label{eq:linregburnin}
    \frac{n}{k} \geq   c_1  
    \left( B_X / \sqrt{\lambda_{\mathrm{min}} (\E XX^\T)}\right)^{3+1/2} \left(\frac{kB_W^2}{\V(W)}\right)
    (d+ \log(1/\delta)) \quad \textnormal{and} \quad k \beta^{-1}(k) \geq n \delta^{-1}
\end{multline}
we have that:
\begin{equation}\label{eq:linregrate}
     \| \widehat f - f_\star\|_{L^2}^2\leq c_2 \V(W)  \left(\frac{d+\log(1/\delta)}{n}\right).
\end{equation}

\end{restatable}

\subsection{Further Comparison to Related Work}

In terms of technical development, this work is most closely related to the work on \iid\ learning in sub-Gaussian classes by \citet{lecue2013learning} and the result for misspecified (agnostic) dependent linear regression by \citet{ziemann2023noise}---which we generalize to more general function classes at the cost of more stringent moment assumptions. Returning to \citet{lecue2013learning}, and beside the fact that they work with independent data, the biggest difference is in how we deal with the multiplier process. We employ chaining with a mixed tail \citep{dirksen2015tail}, instead of a single tail. On a practical level, the advantage of the mixed tail result is that it allows us to push the dependence on, mixing, $L$ (the norm equivalence parameter in \Cref{thm:themainthm}) and any higher order norms into the burn-in. Crucially, we make the observation that chaining with a mixed tail allows us to work with weaker norm relations ($\eta<1$ in \Cref{def:weaksgc}). We do not require equivalence of norms but rather a weaker notion of topological equivalence. Such equivalences hold in significantly wider generality than the sub-Gaussian class assumption as we show in \Cref{sec:examples} below. In particular we are able to handle smoothness classes in \Cref{prop:smoothnessclass}, which cannot be covered in the baseline sub-Gaussian class framework. Another advantage of this approach is that it allows to relegate the parameter $L$ to a higher order term, which appears multiplicatively instead of additively in the bound by \citet{lecue2013learning}. This is important in order to achieve the correct scaling with temporal dependency as there are typically no obvious bounds on this parameter other than in terms of the block-length $k$. Hence, if our dependence on $L$ were multiplicative instead of additive it would thereby re-introduce the sample-size deflation we sought to sidestep. Again, it is the invocation of the mixed-tail chaining result of \citet{dirksen2015tail} that allows for this.

Another closely related line of work studies parameter identification in auto-regressive models \citep[for an overview, see][]{tsiamis2022statistical,ziemann2023tutorial}. When the noise model is strictly realizable---the variables $W_{1:n}$ form a martingale difference sequence with respect to the filtration generated by $X_{1:n}$---parameter identification is possible at the \iid\ rate even in the absence of mixing \citep{simchowitz2018learning, faradonbeh2018finite, sarkar2019near, kowshik2021near}. Our results do not cover the mixing-free regime as we consider the agnostic setting in which self-normalized martingale arguments 
\citep{pena2009self,abbasi2011improved}
are not available. We consider providing a unified analysis of the martingale and mixing situations an interesting future direction.

More generally, several authors have considered learning under various weak dependency notions. \citet{kuznetsov2017generalization} give generalization bounds in a more general setting using the same blocking technique---due to \citet{yu1994mixing}---used here. Statements similar in spirit can also be found in e.g., \citet{steinwart2009fastlearningmixing}, \citet{duchi2012ergodicmd} and most recently \citet{roy2021dependent}. However, they all suffer the dependency deflation discussed above and in our introduction (\Cref{sec:intro}). We also note that \citet{ziemann2022learning} and \citet{maurer2023generalization} obtain rates---similar to ours here---that relegate mixing times into additive burn-in factors. On the one hand, the work of \citet{ziemann2022learning} operates at a similar level of generality when it comes to hypothesis classes and also relies on the square loss function but requires a stringent realizability assumption to be applicable. Moreover, both our noise term and our complexity parameter are sharper than theirs. On the other hand, the work of \citet{maurer2023generalization} operates at a higher level of generality than us, but does not seem to be able to reproduce sharp rates when specialized to our situation.

\section{Examples of Weakly sub-Gaussian Classes}\label{sec:examples}

We conclude by collecting a few examples
of weakly sub-Gaussian classes (\Cref{def:weaksgc}).
Arguably the most compelling example identified in the present manuscript are smoothness classes, which are not covered even in the \iid\ setting by \citet{lecue2013learning}.

\begin{proposition}[Smoothness Classes]
\label{prop:smoothnessclass}
    Let $\sfX$ be a measurable, open, connected and bounded subset of $\R^d$ with Lipschitz boundary and let $\scrF$ be a set of uniformly bounded  functions $f : \sfX \to \R$. Fix an integer $s\in \N$ and suppose that there exists a constant $C_{\scrF}$ such that $ \sum_{|\alpha|\leq s}\| D^\alpha f\|_{L^\infty}\leq C_{\scrF}$.\footnote{Summation over $D^\alpha$ uses multi-index notation---the sum is over all partial derivative operators of order less than or equal to $s$.} Suppose further that the distribution of the covariates $\sfP_X$ has density $\mu_{\sfX}$ with respect to the Lebesque measure and that there exists $\underline{\mu},\overline{\mu}\in \R_+$ such that $\underline{\mu} \leq \mu_{\sfX} \leq \overline{\mu}  $.  Under the above hypotheses there exists a positive constant $c$ only depending on $\sfX$, $d$ and $s$ such that $\scrF$ is $(L,\eta)$-$\Psi_\infty$ with $L = c \overline{\mu}\underline{\mu}^{-\frac{2s}{2s+d}} C_{\scrF}^{\frac{d}{2s+d}}$ and $\eta=\frac{2s}{2s+d}$.
\end{proposition}

\begin{proof}
   The result for $\sfP_X$ equal to the (normalized) Lebesque measure is immediate by the main result of \cite{nirenberg1959elliptic} instantiated to the correct smoothness class. The general case follows by our hypothesis that $\sfP_X$ is equivalent to the Lebesque measure.
\end{proof}

We stress that the quantities $L$ and $\eta$ only appear in the burn-in of \Cref{thm:themainthm}. In other words, \Cref{thm:themainthm} provides sharp rates almost universally, or at least as long as the hypothesis class is sufficiently smooth and bounded (although the latter can be relaxed). However, one important caveat is that this burn-in can be exponentially large (curse of dimensionality) unless the class is sufficently smooth: $s$ is proportional to $d$ above.

Our next example relies on smoothness in parameter space instead of smoothness in terms of inputs.

\begin{proposition}\label{prop:sharpnessprop}
    Fix an open parameter set $\sfM\subset \R^{d_{\scrF}}$ equipped with the Euclidean norm $\|\cdot\|$. Consider a function $\phi : \sfX \times \sfM \to \R$ that generates a parametric class of functions $\scrF=\{ \phi(\cdot; \theta) \  \vline\ \theta \in \sfM \}$. Define  $\sfM_\star \triangleq \argmin_{\theta \in \sfM} \E(\phi(X,\theta)-Y)^2 $ to be the set of population risk minimizers. Suppose that:
    \begin{enumerate}
        \item[(i)] for $a,b\in \R_+$, the estimation error functional of the model $\scrF$ is $(a,b)$-sharp, that is: $\forall \theta \in \sfM$ there exists $\theta_\star \in \sfM_\star$ such that $ ab^{-1}  \|\theta-\theta_\star\|\leq  \left(\E (\phi(X,\theta)-\phi(X,\theta_\star))^2\right)^b$;
        \item[(ii)] the partial gradient $\nabla_\theta \phi(x,\theta)$ exists and is uniformly norm-bounded by $C>0$ for all $(x,\theta) \in \sfX \times \sfM$.
    \end{enumerate}
    Then  $\scrF-\{f_\star\}$ is $(Cba^{-1},2b)$-$\Psi_\infty$.
\end{proposition}

The sharpness condition (i) in \Cref{prop:sharpnessprop} is standard in optimization \citep[see e.g][]{roulet2017sharpness}. This condition holds somewhat generically \citep{lojasiewicz1993geometrie}, but the exact constants $a$ and $b$ are typically difficult to obtain. Fortunately, downstream use of \Cref{prop:sharpnessprop} only relies on these constants in the burn-in.

\begin{proof}
    By the mean value form of Taylor's Theorem and Cauchy-Schwarz we write for fixed $x\in \sfX$:
    \begin{equation}\label{eq:sharpnessexample1}
        |\phi(x;\theta) - \phi(x;\theta_\star)| =| \langle \nabla_\theta \phi(x,\tilde \theta), \theta-\theta_\star \rangle|
        \leq \| \nabla_\theta \phi(x,\tilde \theta)\| \|\theta-\theta_\star\| \leq C \|\theta-\theta_\star\|
    \end{equation}
    for some $\tilde \theta \in [\theta,\theta_\star]$.   
    Consequently by our sharpness hypothesis and by optimizing over the left hand side of \eqref{eq:sharpnessexample1} we have that for some $\theta_\star \in \sfM_\star$ and every $\theta\in \sfM$:
    \begin{equation}
        \sup_{x\in \sfX}\| \phi(x,\theta)-\phi(x,\theta_\star)\| \leq \frac{Cb}{a} \left(\E (\phi(X,\theta)-\phi(X,\theta_\star))^2\right)^b
    \end{equation}
    Equivalently, $\|f\|_{L^\infty} \leq Cba^{-1} \|f\|_{L^2}^{2b}$ for every $f\in \scrF-\{f_\star\}$ as per requirement.
\end{proof}

There is also a more direct argument that easily covers linear functionals on $\R^d$.

\begin{proposition}\label{prop:linreg}
    Let $X$ be a sub-Gaussian random variable taking values in $\R^d$ and let $\scrF$ be the class of linear functionals on $\R^d$. Suppose that $\lambda_{\min} \left(\E XX^\T \right)>0$. Then $\scrF$ is $(L,1)$-$\Psi_2$ with $L = \sup_{v\in \R^d: \|v\|=1 } \frac{\|\langle v, X \rangle\|_{\Psi_2} }{\|\langle v, X \rangle\|_{L^2}} $.
\end{proposition}
\begin{proof}
 The only observation we need to make is that it suffices to prove the result for $\|v\|=1$ by homogeneity. The result is then immediate by construction.
\end{proof}

Analogously, finite hypothesis classes are also covered.
\begin{proposition}\label{prop:finhyp}
    Let $\scrF$ be a finite subset of $L_{\Psi_2}$. Then $\scrF$ is $(L,1)$-$\Psi_2$ with $L=\max_{f\in \scrF} \frac{\|f\|_{\Psi_2}}{\|f\|_{L^2}}$.
\end{proposition}
\begin{proof}
    The result is immediate since the maximum in the quantity $L$ above is achieved since $|\scrF|<\infty$.
\end{proof}



\section{Summary}

In this work, we obtain instance-optimal convergence rates for learning with the square loss function and dependent data. We overcome the typical deflation, by the mixing time, of the sample size. The main technical step to arrive at this result is a refined analysis of the multiplier process \eqref{eq:multiplierdefined} via mixed tail generic chaining that is suitable for dependent, $\beta$-mixing, random variables. Indeed, the leading order term of our main result, \Cref{thm:themainthm}, does not directly depend on any mixing-time type quantities. It mimics the correct asymptotic rate and scales solely in terms of  the statistics of order $2q$ of the process at hand (where typically $q=1+o(1)$). Finally, our result also allows us to evaluate said multiplier process for a wider range of hypothesis classes. Typically, sharp closed form expressions for this process are only available for linear functionals, covered in the \iid\ setting by \citet{lecue2013learning} and \citet{oliveira2016lower}, and extended to the $\beta$-mixing setting by \citet{ziemann2023noise}. By contrast, since our result relies on a weaker notion of topological equivalence, it is applicable to more general classes, such as smoothness classes (\Cref{prop:smoothnessclass}) and parametric classes with sufficiently regular parameterization (\Cref{prop:sharpnessprop}).


\section*{Acknowledgements}

Ingvar Ziemann is supported by a Swedish Research Council international postdoc grant. Nikolai Matni is supported in part by NSF award CPS-2038873, NSF award SLES-2331880, AFOSR Award FA9550-24-1-0102 and NSF CAREER award ECCS-2045834. George J. Pappas acknowledges support from NSF award EnCORE-2217062.
\addcontentsline{toc}{section}{References}

\bibliographystyle{icml2024}
\bibliography{main.bib}

\onecolumn
\appendix

\section{Properties of $\Psi_p$- and $L^p$-Norms}

We begin with an elementary property.

\begin{lemma}\label{lem:productofpsi}
    For every two random variables $Z,Z' \in L_{\Psi_p}$ we have that:
    \begin{equation}
        \| \langle Z, Z' \rangle \|_{\Psi_{p/2}} \leq 2^{2/p}  \| Z\|_{\Psi_p} \| Z' \|_{\Psi_p}.
    \end{equation}
\end{lemma}

\begin{proof}
    We compute:
    \begin{equation}
        \begin{aligned}
             \| \langle Z, Z' \rangle \|_{\Psi_{p/2}} &= \sup_{m\geq 1} \frac{\| \langle Z, Z' \rangle \|_{L^m}}{m^{2/p}}  \\
             &
             =
             \sup_{m\geq 1} \frac{(\E |\langle Z, Z' \rangle|^m)^{1/m} }{m^{2/p}} \\
             &
             \leq
             \sup_{m\geq 1} \frac{(\E \|Z\|^m \| Z'\| ^m)^{1/m} }{m^{2/p}}  && (\textnormal{$\langle \cdot,\cdot \rangle$-Cauchy-Schwarz})\\
             &
             \leq 
             \sup_{m\geq 1} \frac{(\E \|Z\|^{2m} \E \| Z'\| ^{2m})^{1/2m} }{m^{2/p}}  && (L^2\textnormal{-Cauchy-Schwarz})\\
             &
             \leq  2^{2/p}
              \sup_{m\geq 1} \frac{(\E \|Z\|^{2m} )^{1/2m} }{(2m)^{1/p}}  \sup_{m\geq 1} \frac{(\E \| Z'\| ^{2m})^{1/2m} }{(2m)^{1/p}}\\
              &
              \leq 2^{2/p}  \| Z\|_{\Psi_p} \| Z' \|_{\Psi_p}, && (\{ m \geq 1\} \subset \{2m \geq 1\} )
        \end{aligned}
    \end{equation}
    as was required.
\end{proof}


\subsection{Proof of \Cref{lem:phibernsteinmgf}}

\phibernsteinmgf*

\begin{proof}
The idea of the proof is very much the same as that of the standard Bernstein MGF bound but with the modification made in \citet{maurer2021concentration} by which the $L^\infty$ norm is replaced by a $\Psi_p$-norm. We begin by expanding the exponential function:

    \begin{equation}\label{eq_lem:phibernsteinholder}
        \begin{aligned}
            \E \exp(\lambda Z)  &= \E \left[ \sum_{m=0}^\infty  \frac{(\lambda Z)^m}{m!} \right]\\
            &\leq 1 +  \sum_{m=0}^\infty  \frac{\E \left[ (\lambda Z)^2 (\lambda Z)^m\right] }{(m+2)!} && (\E Z\leq 0 )\\
            &\leq 1 +  \lambda^2\left(\E ( Z)^{2q}\right)^{1/q}\sum_{m=0}^\infty  \frac{ \left(\E \left[  |\lambda Z|^{mq'}\right]\right)^{1/q'}}{(m+2)!}. &&(\textnormal{Hölder's Ineq.})
        \end{aligned}
    \end{equation}
We next have:
\begin{equation}\label{eq_lem:phibernsteinstirling}
    \begin{aligned}
        \left(\E \left[   |Z|^{mq'}\right]\right)^{1/q'} &= \|Z\|_{L^{mq'}}^m\\
        &\leq (mq')^{m/p} \| Z\|_{\Psi_p}^m && (\textnormal{df. of } \Psi_p)\\
        &\leq  (m!)^{1/p} (q' e)^{m/p} \| Z\|_{\Psi_p}^m. && (\textnormal{Stirling's Approximation})
    \end{aligned}
\end{equation}
Upon combining \eqref{eq_lem:phibernsteinholder} with \eqref{eq_lem:phibernsteinstirling} we arrive at
\begin{equation}
    \begin{aligned}
        &\E \exp(\lambda Z)\\ &\leq 1+ \lambda^2\left(\E ( Z)^{2q}\right)^{1/q}\sum_{m=0}^\infty  \frac{ (m!)^{1/p} \lambda^m (q' e)^{m/p} \| Z\|_{\Psi_p}^m}{(m+2)!}\\
        &
        \leq 1+ \frac{\lambda^2\left(\E ( Z)^{2q}\right)^{1/q}}{2}\sum_{m=0}^\infty  \left(\lambda (q' e)^{1/p} \| Z\|_{\Psi_p}\right)^m && \left(p\in [1,\infty], m\in \N \Rightarrow \frac{(m!)^{1/p}}{(m+2)!}\leq \frac{1}{2} \right)\\
        &
        =1 +\frac{\frac{\lambda^2}{2}\left(\E ( Z)^{2q}\right)^{1/q}}{1-\lambda (q' e)^{1/p} \| Z\|_{\Psi_p}} && \left(x\in [0,1) \Rightarrow \sum_{m=0}^\infty x^m = \frac{1}{1-x}\right)\\
        &\leq \exp \left( \frac{\frac{\lambda^2}{2}\left(\E ( Z)^{2q}\right)^{1/q}}{1-\lambda (q' e)^{1/p} \| Z\|_{\Psi_p}}\right), && \left(x\in \R \Rightarrow 1+x \leq e^x\right)
    \end{aligned}
\end{equation}
as per requirement.
\end{proof}

\section{Controlling the Multiplier Process}
\label{sec:multiplierapp}
\bernsteinemppt*

\begin{proof}
First, note that we may assume throughout the proof that $\|g-g'\|_{L^2}>0$, for otherwise the result is trivial. We now begin by applying \Cref{lem:phibernsteinmgf}:  
    \begin{equation}
    \begin{aligned}
        &\E \exp \left(\lambda\sum_{i=1}^k (1-\E) \langle W_i, g(X_i)-g'(X_i)\rangle \right)\\
        &
        \leq
        \exp \left( \frac{\lambda ^2k \|g-g'\|_{L^2}^2  \V_{2q}\left( \frac{1}{\sqrt{k}\|g-g'\|_{L^2}}\sum_{i=1}^k  \langle W_i, g(X_i)-g'(X_i) \rangle \right)  }{2\left(1-\lambda (q' e)^{2/p} \| \sum_{i=1}^k (1-\E)  \langle W_i, g(X_i)-g'(X_i) \rangle\|_{\Psi_{p/2}}  \right)}   \right)
    \end{aligned}
    \end{equation}
    as long as $\lambda < \left( (q' e)^{2/p} \|  \sum_{i=1}^k (1-\E)  \langle W_i, g(X_i)-g'(X_i) \rangle\|_{\Psi_{p/2}} \right)^{-1} $.        
    Now, by triangle inequality and \Cref{lem:productofpsi},
    \begin{equation}
    \begin{aligned}
        \lambda (q' e)^{2/p} \left\|  \sum_{i=1}^k (1-\E)  \langle W_i,g(X_i)-g'(X_i) \rangle\right\|_{\Psi_{p/2}} 
        &\leq
        \lambda (2 q' e)^{2/p} k \|  W \|_{\Psi_{p}} \| g(X_i)-g'(X_i)\|_{\Psi_p}.
    \end{aligned}
    \end{equation}
    Consequently:
    \begin{equation}
        \begin{aligned}
             &\E \exp \left(\lambda\sum_{i=1}^k (1-\E) \langle W_i, g(X_i)-g'(X_i)\rangle \right)\\
        &
        \leq 
        \exp \left( \frac{\lambda ^2k \|g-g'\|_{L^2}^2  \V_{2q}\left(\{g\}-\{g'\} \right)  }{2\left(1- \lambda (2q' e)^{2/p} k \|  W \|_{\Psi_{p}} \| g-g'\|_{\Psi_p}  \right)}   \right).
        \end{aligned}
    \end{equation}
Since the process is $k$-wise independent and mean zero we thus have that:
 \begin{equation}
        \begin{aligned}
             &\E \exp \left(\lambda\sum_{i=1}^n (1-\E) \langle W_i, g(X_i)-g'(X_i)\rangle \right)\\
        &
        \leq 
        \exp \left( \frac{\lambda ^2 n \|g-g'\|_{L^2}^2  \V_{2q}\left(\{g\}-\{g'\} \right)  }{2\left(1- \lambda (2q' e)^{2/p}  k \|  W \|_{\Psi_{p}} \| g-g'\|_{\Psi_p}  \right)}   \right).
        \end{aligned}
    \end{equation}
Hence for every $\lambda \in \left[0, \left(2(2q' e)^{2/p} k \|  W \|_{\Psi_{p}} \| g-g'\|_{\Psi_p}  \right)^{-1} \right)\triangleq \Lambda$ we have:
\begin{equation}
        \begin{aligned}
             &\E \exp \left(\lambda\sum_{i=1}^n (1-\E) \langle W_i, g(X_i)-g'(X_i)\rangle \right)\\
        &
        \leq 
        \exp \left( \lambda ^2 n \|g-g'\|_{L^2}^2  \V_{2q}\left(\{g\}-\{g'\} \right)    \right).
        \end{aligned}
    \end{equation}
Taking the above exponential inequality as a starting point, for a fixed $u\in (0,\infty)$, a Chernoff argument now yields:
\begin{equation}
    \begin{aligned}
        &\Pr \left( \sum_{i=1}^n (1-\E) \langle W_i, g(X_i)-g'(X_i)\rangle > u\right) \\
        &\leq \inf_{\lambda > 0} \E \exp \left(-u\lambda + \lambda \sum_{i=1}^n (1-\E) \langle W_i, g(X_i)-g'(X_i)\rangle \right) \\
        &\leq \inf_{\lambda \in \Lambda}
        \left( -\lambda u +\frac{\lambda ^2 n \|g-g'\|_{L^2}^2  \V_{2q}\left(\{g\}-\{g'\} \right)  }{2\left(1- \lambda (2q' e)^{2/p}  k \|  W \|_{\Psi_{p}} \| g-g'\|_{\Psi_p}  \right)}   \right)
        \\
        &\leq 
        \begin{cases}
            \exp\left( \frac{-u^2}{4 n\left( \|g-g'\|_{L^2}^2  \V_{2q}\left(\{g\}-\{g'\} \right)\right)}\right)
            & u\leq  \frac{ \left( n \|g-g'\|_{L^2}^2  \V_{2q}\left(\{g\}-\{g'\}\right)\right)}{2 (2q' e)^{2/p}   k \|  W \|_{\Psi_{p}} \| g-g'\|_{\Psi_p}  }, 
            \\
            \exp \left(\frac{-u}{4 (2q' e)^{2/p}  k \|  W \|_{\Psi_{p}} \| g-g'\|_{\Psi_p}  } \right)
            &\textnormal{ otherwise.}
        \end{cases}
    \end{aligned}
\end{equation}
Rescaling and summing the failure probabilities in either case yields:
\begin{multline}
    \Pr \Bigg( \sum_{i=1}^n (1-\E) \langle W_i, g(X_i)-g'(X_i)\rangle \\
    > \sqrt{4 n\|g-g'\|_{L^2}^2  \V_{2q}\left(\{g\}-\{g'\} \right) u}\\
    + 
    4 (2q' e)^{2/p}  k \|  W \|_{\Psi_{p}} \| g-g'\|_{\Psi_p} u \Bigg)
    \leq 2e^{-u},
\end{multline}
as was required.
\end{proof}

Let us turn to making \Cref{lem:bernsteinemppt} uniform. By instantiating Theorem 3.5 of \cite{dirksen2015tail} combined with the  pointwise control of \Cref{lem:bernsteinemppt}, we  immediately have the following result.

\begin{lemma}[Corollary of Theorem 3.5 in \cite{dirksen2015tail}]
\label{thm:dirksens}
    Fix $\delta \in (0,1),r>0$ and consider the space $\scrF_{\star}\cap rS_{L^2}$ endowed with the metrics
    \begin{equation}
        \begin{aligned}
            d_1(g,g')&=4 (q' e)^{2/p}  k \|  W \|_{\Psi_{p}} \| g-g'\|_{\Psi_p},\\
            d_2(g,g')&=\sqrt{4 n\left(   \V_{2q}\left(\scrF_{\star}\cap rS_{L^2}\right)\right) }\|g-g'\|_{L^2},
        \end{aligned}
    \end{equation}
    and denote the corresponding diameters by $\Delta_i(\scrF_{\star}\cap rS_{L^2}) \triangleq \sup_{g,g'\in \scrF_{\star}\cap rS_{L^2}}d_i(g,g') ,i\in[2]$.    There exist universal positive constants $c_1$ and $c_2$ such that with probability at least $1-\delta$:
    \begin{align}
         \sup_{f\in \scrF_{\star}\cap rS_{L^2}}\sum_{i=1}^n (1-\E) \langle W_i, f\rangle 
         &\leq c_1 \left( \gamma_2(\scrF_{\star}\cap rS_{L^2},d_2) +\gamma_1(\scrF_{\star}\cap rS_{L^2},d_1) \right) \nonumber \\
         &\qquad+c_2\left(\Delta_2(\scrF_{\star}\cap rS_{L^2})\sqrt{\log (1/\delta)} + \Delta_1(\scrF_{\star}\cap rS_{L^2}) \log(1/\delta) \right). \label{eq:dirksencorr}
    \end{align}
\end{lemma}

We now restate and prove the result for the multiplier process.

\multiplieruniform*

\begin{proof}
    We need to translate the metrics (appearing in \Cref{thm:dirksens}) $d_1,d_2$ and their diameters into the standard $L^2$-metric using that the class is $(L,\eta)$-$\Psi_p$. We begin with $d_2$, which is just a dilated $L^2$-metric:
    \begin{equation}
        \begin{aligned}
            \gamma_2(\scrF_{\star}\cap rS_{L^2},d_2) 
            &=  \inf_{\{F_m\}} \sup_{f\in \scrF_\star\cap rS_{L^2}} \sum_{m=0}^{\infty} 2^{m/2} d_2(f,F_m)  \\  
            &=\sqrt{4 n\left(   \V_{2q}\left(\scrF_{\star}\cap rS_{L^2}\right)\right) }\gamma_2(\scrF_{\star}\cap rS_{L^2},d_{L^2}),
        \end{aligned}
    \end{equation}
    and
    \begin{equation}
        \Delta_2(\scrF_{\star}\cap rS_{L^2}) \leq r  \sqrt{4 n\left(   \V_{2q}\left(\scrF_{\star}\cap rS_{L^2}\right)\right) }.
    \end{equation}
    Turning to the $d_1$, we have:
  \begin{align*}
            \gamma_1(\scrF_{\star}\cap rS_{L^2},d_1) 
            &=  \inf_{\{F_m\}} \sup_{f\in \scrF_\star\cap rS_{L^2}} \sum_{m=0}^{\infty} 2^{m} d_1(f,F_m)  \\  
            &=
            \inf_{\{F_m\}} \sup_{f\in \scrF_\star\cap rS_{L^2}} \sum_{m=0}^{\infty} 2^{m} 2 (q' e)^{2/p} k \|  W \|_{\Psi_{p}} d_{\Psi_p}(f,F_m)\\
            &\leq  2 (q' e)^{2/p} L  k \|  W \|_{\Psi_{p}}  \gamma_1(\scrF_{\star}\cap rS_{L^2},d_{L^2}^\eta)\\
             &\leq  2 (q' e)^{2/p} L  k \|  W \|_{\Psi_{p}}  \gamma_\eta(\scrF_{\star}\cap rS_{L^2},d_{L^2}),
        \end{align*}
        where the first inequality uses that  $\scrG$ is $(L,\eta)$-$\Psi_p$ and the last inequality uses that $ \gamma_1(\scrF_{\star}\cap rS_{L^2},d_{L^2}^\eta) \leq  \gamma_\eta(\scrF_{\star}\cap rS_{L^2},d_{L^2})$ as long as $r\leq 1$. Similarly:
        \begin{equation}
            \Delta_1(\scrF_{\star}\cap rS_{L^2}) \leq 2 (q' e)^{2/p}  L  k \|  W \|_{\Psi_{p}}r^\eta.
        \end{equation}
The result follows by substituting the above expressions into the result of \cite{dirksen2015tail} captured as \Cref{thm:dirksens}.
\end{proof}

\section{Controlling the Quadratic Process}

\label{sec:quadapp}
\begin{lemma}[Truncation Accuracy]\label{lem:trunclem}
Fix $\e,r>0$, let $\scrG$ be $(L,\eta)$-$\Psi_p$, and let $g\in \scrG$ be such that $\|g\|_{L^2}=r$. For $\tau \in \R_+$, define $g_\tau \triangleq g \mathbf{1}_{\|g\|\leq \tau}$. There exists a truncation level $\tau$ and a universal positive constant $c>0$ such that:
    \begin{equation}
        \|g\|^2_{L^2} -\|g_\tau\|^2_{L^2} \leq r^2 \e
    \end{equation}
and
\begin{equation}
            \tau \leq  Lr^\eta \left( c^{-1} \log \left(\frac{4^{2/p}L^{1/2}r^{2\eta} }{\e r^4} \right)  \right)^{1/p}.
\end{equation}
    
\end{lemma}

\begin{proof}
    Fix a level $\tau>0$ to be determined later. For any such level we have that:
    \begin{equation}
        \begin{aligned}
            \|g\|^2_{L^2} -\|g_\tau\|^2_{L^2} \leq \E [\|g \|^2 \mathbf{1}_{\|g\|> \tau}]\leq \sqrt{\E \|g\|^4 \mathbf{P}(\|g\| >\tau) } \leq \sqrt{\E \|g\|^4  } \exp\left(-c \tau^p/\|g \|_{\Psi_p}^p  \right).
        \end{aligned}
    \end{equation}
    Hence if we choose $\tau^p =c^{-1}\|g \|_{\Psi_p}^p  \log\left(\frac{\sqrt{E\|g\|^4}}{\e r^2\E \|g\|^2} \right) $ we have:
    \begin{equation}
        \|g\|^2_{L^2} -\|g_\tau\|^2_{L^2} \leq \e^2.
    \end{equation}

    It remains to derive an upper bound on $\tau$. Since $\scrG$ is $(L,\eta)$-$\Psi_p$ and $\|g\|_{L^2}=r$ we have that
    \begin{equation}
    \begin{aligned}
        \| g\|_{\Psi_p}& \leq L \|g\|_2^\eta = Lr^\eta, \textnormal{ and}
        \\
        \E \| g\|^4 &\leq 4^{4/p} \| g\|_{\Psi_p}^4  \leq 4^{4/p} Lr^{4\eta}.
    \end{aligned}
    \end{equation}
    Hence our choice of $\tau$ satisfies:
    \begin{equation}
        \tau \leq  Lr^\eta \left( c^{-1} \log \left(\frac{4^{2/p}L^{1/2}r^{2\eta} }{\e r^4} \right)  \right)^{1/p}
    \end{equation}
    and so the result has been established.
\end{proof}

\loweruniformlaw*

\begin{proof}
    By star-shapedness, we may assume without loss of generality that $f \in \scrF_\star \cap rS_{L^2}$. Fix $\tau$ such that for all such $f$
    \begin{equation}
        \|f\|^2_{L^2} -\|f_\tau\|^2_{L^2} \leq \e^2,
    \end{equation}
    and note that the existence of such a level is guaranteed by \Cref{lem:trunclem}.
    
    It is then clear that:
    \begin{equation}\label{eq_thm:loweruniform_truncemployed}
        \frac{1}{n}\sum_{i=1}^n \|f(X_i)\|^2 \geq  r^2(1-\e^2)  - \sup_{f\in \scrF_\star \cap rB_{L^2}} \left\{\frac{1}{n}\sum_{i=1}^n \|f_\tau(X_i)\|^2 - \|f_\tau\|_{L^2}^2   \right\}
    \end{equation}
    and, for well-chosen $\e,r$, it therefore suffices to control the supremum of empirical process to the right of \eqref{eq_thm:loweruniform_truncemployed} and we will use Dirksen's theorem again to do so \citep[][Theorem 3.5]{dirksen2015tail}. A preliminary estimate using $k$-wise independence, stationarity  and \Cref{lem:phibernsteinmgf} gives that for every $f,g$ and admissible $\lambda$:
    \begin{multline}\label{eq_thm:loweruniformprelestimate}
            \E \exp \left( \lambda \left[ \sum_{i=1}^n \|f_\tau(X_i)\|^2 - \|f_\tau\|_{L^2}^2- \|g_\tau(X_i)\|^2 + \|g_\tau\|_{L^2}^2 \right]\right) 
            \\
            \leq \exp\left( \frac{\lambda^2 n\V_{4}\left( \frac{1}{\sqrt{k}}\sum_{i=1}^k \|f_\tau(X_i)\|^2 - \|f_\tau\|_{L^2}^2- \|g_\tau(X_i)\|^2 + \|g_\tau\|_{L^2}^2 \right) }{2 \left( 1-\lambda (q' e)^{1/p} k \left\|\|f_\tau(X_i)\|^2 - \|f_\tau\|_{L^2}^2- \|g_\tau(X_i)\|^2 + \|g_\tau\|_{L^2}^2 \right\|_{\Psi_{p/2}} \right) } \right).    
    \end{multline}
    This is almost the exponential inequality we need, but we will want increment conditions for the above empirical process in terms $\Psi_p$ and $L^2$.

    The increment condition in $\Psi_p$ is simple. We observe that for any two $f,g$ and any $x$ in their domain:
    \begin{equation}\label{eq_thm:loweruniformusefulrewrite}
        \begin{aligned}
            \|f_\tau(x)\|^2-\|g_\tau(x)\|^2 = \langle (f_\tau+g_\tau)(x),   (f_\tau+g_\tau)(x) \rangle.
        \end{aligned}
    \end{equation}
Consequently by  \Cref{lem:productofpsi} and $\tau$-truncation:
\begin{equation}
    \| \|f_\tau(X)\|^2-\|g_\tau(X)\|^2 \|_{\Psi_{p/2}} \leq 2^{2/p} \| f_\tau+g_\tau \|_{\Psi_p} \| f_\tau-g_\tau\|_{\Psi_p} \leq 2^{1+2/p}\tau    \| f-g\|_{\Psi_p}.
\end{equation}
A centering argument thus gives:
\begin{equation}
    \| \|f_\tau(X)\|^2-\|f_\tau\|_{L^2}^2-\|g_\tau(X)\|^2+\|g_\tau\|_{L^2}^2 \|_{\Psi_{p/2}} \leq 2^{2+2/p}\tau   \| f-g\|_{\Psi_p}
\end{equation}
wherefore we set 
\begin{equation}\label{eq_thm:loweruniformorld1}
  d_1(f,g) \triangleq  (q' e)^{1/p} k 2^{2+2/p}\tau   \| f-g\|_{\Psi_p}.
\end{equation}

Let us next address the variance term:
\begin{equation}\label{eq_thm:loweruniformvar}
    \begin{aligned}
        &\V_{4}\left( \frac{1}{\sqrt{k}}\sum_{i=1}^k \|f_\tau(X_i)\|^2 - \|g_\tau(X_i)\|^2 \right) \\
        &=  \V_{4}\left( \frac{1}{\sqrt{k}}\sum_{i=1}^k \langle f_\tau(X_i)+g_\tau(X_i),f_\tau(X_i)-g_\tau(X_i)\rangle   \right) &&(\textnormal{use }\eqref{eq_thm:loweruniformusefulrewrite}) \\
        &\leq \sqrt{\E\left( \frac{1}{\sqrt{k}}\sum_{i=1}^k \langle f_\tau(X_i)+g_\tau(X_i),f_\tau(X_i)-g_\tau(X_i)\rangle   \right)^4}
        && (\V_4[\cdot] \leq \sqrt{\E[| \cdot|^4]} ) 
        \\
        &\leq k \sqrt{\E\left( \langle f_\tau(X)+g_\tau(X),f_\tau(X)-g_\tau(X)\rangle   \right)^4}
        &&(\textnormal{Cauchy-Schwarz})\\
        &\leq 4k\tau^2 \|f-g\|_{L^4}^2\triangleq d^2_2(f,g). &&(\tau\textnormal{-boundedness and Cauchy-Schwarz})
    \end{aligned}
\end{equation}

With $d_1,d_2$ as in \eqref{eq_thm:loweruniformorld1} and \eqref{eq_thm:loweruniformvar}, we can now estimate \eqref{eq_thm:loweruniformprelestimate} as:
\begin{equation}\label{eq_thm:loweruniformmgfestimate}
            \E \exp \left( \lambda \left[ \sum_{i=1}^n \|f_\tau(X_i)\|^2 - \|f_\tau\|_{L^2}^2- \|g_\tau(X_i)\|^2 + \|g_\tau\|_{L^2}^2 \right]\right) 
            \leq  \exp\left( \frac{\lambda^2 n d_2^2(f,g) }{2 \left( 1-\lambda d_1(f,g)\right) } \right).
\end{equation}
We thus obtain the probability estimate ($u > 0$):
\begin{equation}\label{eq_thm:loweruniformdirksensetup}
    \begin{aligned}
        \Pr \left(\frac{1}{n}\sum_{i=1}^n \|f_\tau(X_i)\|^2 - \|f_\tau\|_{L^2}^2- \|g_\tau(X_i)\|^2 + \|g_\tau\|_{L^2}^2 > c' \sqrt{u/n}d_2(f,g) + c (u/n) d_1(f,g)   \right)\leq 2e^{-u}
    \end{aligned}
\end{equation}
for two universal positive constants $c,c'$. After defining (normalizing) for some universal positive constants $c_1,c_2$: 
\begin{align}\label{eq_thm:loweruniformtilded1}
    \tilde d_1(f,g) &\triangleq c n^{-1} d_1(f,g) =c_1 n^{-1}  (q' e)^{1/p} k 2^{2+2/p}\tau   \| f-g\|_{\Psi_p},  \\
    \tilde d_2(f,g) &\triangleq  c' n^{-1/2} d_2(f,g) = c_2 n^{-1/2}\left(k\tau^2 \|f-g\|_{L^4}^2\right),\label{eq_thm:loweruniformtilded2}
\end{align}
we notice that \eqref{eq_thm:loweruniformdirksensetup} is consistent with the mixed tail generic chaining condition in \citet[Equation 12]{dirksen2015tail} for metrics $\tilde d_1, \tilde d_2$. Consequently, by Theorem 3.5 in \cite{dirksen2015tail} we have that:
\begin{multline}
     \sup_{f\in \scrF_\star \cap rB_{L^2}} \left\{\frac{1}{n}\sum_{i=1}^n \|f_\tau(X_i)\|^2 - \|f_\tau\|_{L^2}^2   \right\} 
     \leq 
     c_3(\gamma_2(\scrF_\star \cap rB_{L^2},\tilde d_2) + \sqrt{u} \Delta_{\tilde d_2}(\scrF_\star \cap rB_{L^2}))
     \\
     + c_4(\gamma_1(\scrF_\star \cap rB_{L^2},\tilde d_1)+u\Delta_{\tilde d_1}(\scrF_\star \cap rB_{L^2}) )
\end{multline}
for two universal positive constants $c_3,c_4$.

To finish the proof, we turn to relating the quantities $\gamma$ and $\Delta$ in terms of problem data.  We have  (recalling \eqref{eq_thm:loweruniformtilded1}):
\begin{equation}
    \Delta_{\tilde d_1}(\scrF_\star \cap rB_{L^2}) = c n^{-1}  (q' e)^{1/p} k 2^{2+2/p}\tau \Delta_{\Psi_p}(\scrF_\star \cap rB_{L^2})  \leq  c n^{-1}  (q' e)^{1/p} k 2^{2+2/p}\tau Lr^\eta
\end{equation}
and also (recalling \eqref{eq_thm:loweruniformtilded2}):
\begin{equation}
    \begin{aligned}
        \Delta_{\tilde d_2}(\scrF_\star \cap rB_{L^2})
        &\leq   c' n^{-1/2}2\sqrt{k}\tau\Delta_{ d_{L^2}}(\scrF_\star \cap rB_{L^4}) 
        \\
        &\leq c' n^{-1/2}2\sqrt{k}\tau L^{3/4} r^{\frac{1+3\eta}{4}},
    \end{aligned}
\end{equation}
where we used Cauchy-Schwarz and the class assumption in the last step to control the $L^4$ norm by the $L^2$ norm.

As for $\gamma$-functionals, we have:
 \begin{equation}
            \gamma_1(\scrF_{\star}\cap rS_{L^2},\tilde d_1) 
            \leq  c n^{-1}  (q' e)^{1/p} k 2^{2+2/p}\tau L \gamma_\eta(\scrF_{\star}\cap rS_{L^2}, d_{L^2}) 
 \end{equation}
and 
\begin{equation}
        \gamma_2(\scrF_{\star}\cap rS_{L^2},\tilde d_2) %
        \leq   c' n^{-1/2}2\sqrt{k}\tau L^{3/4}\gamma_{\frac{2+6\eta}{4}}(\scrF_{\star}\cap rS_{L^2},d_{L^2}).
\end{equation}

Putting everything together we thus obtain that:
\begin{align*}
    &\sup_{f\in \scrF_\star \cap rB_{L^2}} \left\{\frac{1}{n}\sum_{i=1}^n \|f_\tau(X_i)\|^2 - \|f_\tau\|_{L^2}^2   \right\} \\
     &\leq  C L^{3/4}n^{-1/2}\sqrt{k}\tau \left( \gamma_{\frac{2+6\eta}{4}}(\scrF_{\star}\cap rS_{L^2},d_{L^2})+ r^{\frac{1+3\eta}{4}} \sqrt{\log (1/\delta)}\right)
     \\
     &\qquad+
     C'2^{C''/p} n^{-1}  (q' )^{1/p} k  \tau L \left( \gamma_\eta(\scrF_{\star}\cap rS_{L^2}, d_{L^2})
     +r^\eta \log(1/\delta) \right)\\
      &= C n^{-1/2}\sqrt{k}    L^{1+3/4}r^\eta \left( c^{-1} \log \left(\frac{4^{2/p}L }{\e r} \right)  \right)^{1/p}    \left( \gamma_{\frac{2+6\eta}{4}}(\scrF_{\star}\cap rS_{L^2},d_{L^2})+ r^{\frac{1+3\eta}{4}} \sqrt{\log (1/\delta)}\right)
     \\
     &\qquad+
     C'2^{C''/p} n^{-1}  (q' )^{1/p} k          r^\eta \left( c^{-1} \log \left(\frac{4^{2/p}L }{\e r} \right)  \right)^{1/p}                    L^2 \left( \gamma_\eta(\scrF_{\star}\cap rS_{L^2}, d_{L^2})
     +r^\eta \log(1/\delta) \right),
\end{align*}
for universal positive constants $C,C',C''$. Since $p\geq 1$ we may replace all the terms containing upper-case universal constants by a single universal constant as in the theorem statement.
\end{proof}

\section{Results for Mixing Empirical Processes}

\subsection{Blocking}
\label{sec:blocking}

Recall that we partition $[n]$ into $2m$ consecutive intervals, denoted $a_j$ for $j \in [2m]$, so that $\sum_{j=1}^{2m} |a_j| = n$. Denote further by $O$ (resp. by $E$) the union of the oddly (resp. evenly) indexed subsets of $[n]$. We further abuse notation by writing $\beta_Z(a_i) = \beta_Z(|a_i|)$ in the sequel.

We split the process $Z_{1:n}$ as:
\begin{equation}\label{eq:blockstructure}
    Z^o_{1:|O|} \triangleq (Z_{a_1},\dots,Z_{a_{2m-1}}), \quad Z^e_{1:|E|} \triangleq (Z_{a_2},\dots,Z_{a_{2m}}).
\end{equation}
Let $\tilde Z^o_{1:|O|}$ and $\tilde Z^e_{1:|E|}$ be blockwise decoupled versions of \eqref{eq:blockstructure}. That is we posit that $\tilde Z^o_{1:|O|} \sim \sfP_{\tilde Z^o_{1:|O|}}$ and $\tilde Z^e_{1:|E|} \sim \sfP_{\tilde Z^e_{1:|E|}}$, where:
\begin{equation}\label{eq:blocktensor}
    \sfP_{\tilde Z^o_{1:|O|}} \triangleq \sfP_{Z_{a_1}} \otimes \sfP_{Z_{a_3}}\otimes  \dots \otimes \sfP_{Z_{a_{2m-1}}} 
    \quad \textnormal{and}\quad
    \sfP_{\tilde Z^e_{1:|E|}} \triangleq \sfP_{Z_{a_2}} \otimes \sfP_{Z_{a_4}} \otimes\dots \otimes \sfP_{Z_{a_{2m}}}.
\end{equation}
The process $\tilde Z_{1:n}$ with the same marginals as $\tilde Z^o_{1:|O|}$ and $\tilde Z^e_{1:|E|}$ is said to be the decoupled version of $Z_{1:n}$. To be clear: $\sfP_{\tilde Z_{1:n}} \triangleq \sfP_{Z_{a_1}} \otimes \sfP_{Z_{a_2}}\otimes  \dots \otimes \sfP_{Z_{a_{2m}}}$, so that $\tilde Z^o_{1:|O|}$ and $\tilde Z^e_{1:|E|}$ are alternatingly embedded in  $\tilde Z_{1:n}$. The following result is key---by skipping every other block, $\tilde Z_{1:n}$ may be used in place of $Z_{1:n}$ for evaluating bounded scalar functionals, such as probabilities of measurable events, at the cost of an additive mixing-related term.

\begin{proposition}[Lemma 2.6 in \cite{yu1994mixing}; Proposition 1 in \cite{kuznetsov2017generalization}]
\label{prop:decoupling}
Fix a $\beta$-mixing process $Z_{1:n}$ and let $\tilde Z_{1:n}$ be its decoupled version.  For any measurable function $f$ of $Z^o_{1:|O|}$ (resp.\ $g$ of $Z^e_{1:|E|}$)
with joint range $[0,1]$ we have that:
\begin{equation}
    \begin{aligned}
        | \E(f(Z^o_{1:|O|})) - \E (f(\tilde Z^o_{1:|O|})) | &\leq \sum_{i \in E\setminus \{2m\}} \beta_Z(a_i),\\
        | \E(g(Z^e_{1:|E|})) - \E (g(\tilde Z^e_{1:|E|})) | &\leq \sum_{i \in O\setminus\{1\}} \beta_Z(a_i).
    \end{aligned}
\end{equation}
\end{proposition}

\subsection{Controlling  Empirical Processes for $\beta$-Mixing Data}
\label{sec:empprocessmix}

Applying \Cref{prop:decoupling} to \Cref{thm:multiplierthm} and \Cref{thm:loweruniform} yields the desired control of the multiplier and quadratic processes also for $\beta$-mixing data.

\begin{proposition}
\label{prop:multipliermix}
      Fix a failure probability $\delta \in (0,1)$, a positive scalar $r\in (0,\infty)$, two Hölder conjugates $q$ and $q'$, and a class $\scrF$. Suppose that $\scrF_\star -\scrF_\star$ is $(L,\eta)$-$\Psi_p$. Suppose further that the model $\sfP_{(X,Y)_{1:n}}$ is stationary and $\beta$-mixing and suppose further that $k\in \N$ divides $n/2$. There exist universal positive constants $c_1,c_2$ such that for any $r\in (0,1]$ we have that with probability at least $1-\delta-\frac{n}{k}\beta(k)$:
    \begin{multline}
         \sup_{f\in \scrF_{\star}\cap rS_{L^2}}\frac{1}{rn}\sum_{i=1}^n (1-\E) \langle W_i, f\rangle \\
         \leq  c_2 \sqrt{    \V_{2q}\left(\scrF_{\star}\cap rS_{L^2}\right) }  \left(\frac{1}{r\sqrt{n}} \gamma_2(\scrF_{\star}\cap rS_{L^2},d_{L^2}) + \sqrt{\frac{\log(1/\delta)}{n}} \right) \\
         +c_1(q' e)^{2/p} L  k \|W\|_{\Psi_p}\left(\frac{1}{rn}\gamma_\eta(\scrF_{\star}\cap rS_{L^2},d_{L^2}) +\frac{r^{\eta-1}}{n} \log(1/\delta)\right)
    \end{multline}

\end{proposition}

\begin{proposition}
  \label{thm:loweruniformmix}

Fix a failure probability $\delta \in (0,1)$, a tolerance $\e>0$, a localization radius $r\in (0,1]$, and two Hölder conjugates $q$ and $q'$. Suppose that $\scrF_\star-\scrF_\star$ is $(L,\eta)$-$\Psi_p$. Suppose further that the model $\sfP_{(X,Y)_{1:n}}$ is stationary and $\beta$-mixing and suppose further that $k\in \N$ divides $n/2$.  There exists a universal positive constant $c$ such that uniformly for all $f\in \scrF_\star \cap rS_{L^2}$ we have that with probability at least $1-\delta-\frac{n}{k}\beta(k)$:

\begin{multline}
\frac{1}{n}\sum_{i=1}^n \|f(X_i)\|^2 \geq  r^2(1-\e^2)  \\- c \Bigg \{ n^{-1/2}\sqrt{k}    L^{1+3/4}r^\eta \left(  \log \left(\frac{4^{2/p}L }{\e r} \right)  \right)^{1/p} \left( \gamma_{\frac{2+6\eta}{4}}(\scrF_{\star}\cap rS_{L^2},d_{L^2})+ r^{\frac{1+3\eta}{4}} \sqrt{\log (1/\delta)}\right)
     \\
     +
      n^{-1}  (q' )^{1/p} k          r^\eta \left(  \log \left(\frac{4^{2/p}L }{\e r} \right)  \right)^{1/p}                    L^2  \left( \gamma_\eta(\scrF_{\star}\cap rS_{L^2}, d_{L^2})
     +r^\eta \log(1/\delta) \right) \Bigg\}.
\end{multline}
  
\end{proposition}

\section{Finishing the Proof of \Cref{thm:themainthm}}

Before we finish the proof of the main result, let us first make formal the justification for the introduction of quadratic and multiplier processes in \Cref{sec:contribution}. The following lemma bounds the excess risk of empirical risk minimizer in terms of these.

\begin{restatable}[Localized Basic Inequality]{lemma}{lembasicineq}
\label{lem:basicineq}
Suppose that either (1) $\scrF$ is convex or (2) $\scrF$ is realizable. For every $r>0$ we have that:
\begin{equation}\label{eq_lem:basicineq}
        \|\widehat f-f_\star\|_{L^2}^2
        \leq r^2 +\frac{1}{r^2}\left(\sup_{g\in \scrF_\star\cap r S_{L^2}} M_n(g)\right)^2 + \sup_{g\in \scrF_\star} Q_n(g).
\end{equation}
\end{restatable}

\begin{proof}
We begin by observing that the optimality of $\widehat f$ to \eqref{eq:ermdef} yields the basic inequality:
\begin{equation}
     \frac{1}{n}\sum_{i=1}^n\|\widehat f(X_i)-f_\star(X_i)\|^2 \leq \frac{2}{n}\sum_{i=1}^n  \langle W_i, (\widehat f-f_\star)(X_i)\rangle.  
\end{equation}

If $\scrF$ is convex, we have that $\E \langle W_i, ( f-f_\star)(X_i)\rangle  \leq 0$ for every $f$ (by optimality of $f_\star$ to the population objective). If instead $\scrF$ is realizable the same holds true but with equality. Hence, in either case:
\begin{equation}
     \frac{1}{n}\sum_{i=1}^n\|\widehat f(X_i)-f_\star(X_i)\|^2 \leq \frac{2}{n}\sum_{i=1}^n  (1-\E')\langle W_i, (\widehat f-f_\star)(X_i)\rangle  
\end{equation}
where $\E'$ denotes expectation with respect to a fresh copy of randomness (independent of the data used to construct $\widehat f)$.

Consequently we also have that:
\begin{multline}
\label{eq:thesplitismade}
    \|\widehat f-f_\star\|_{L^2}^2 =   \frac{(1+\e)}{n}\sum_{i=1}^n\|\widehat f(X_i)-f_\star(X_i)\|^2 +\|\widehat f-f_\star\|_{L^2}^2 - \frac{(1+\e))}{n}\sum_{i=1}^n\|\widehat f(X_i)-f_\star(X_i)\|^2 
    \\
    \leq \frac{2(1+\e)}{n}\sum_{i=1}^n  (1-\E')\langle W_i, (\widehat f-f_\star)(X_i)\rangle +\|\widehat f-f_\star\|_{L^2}^2 - \frac{(1+\e))}{n}\sum_{i=1}^n\|\widehat f(X_i)-f_\star(X_i)\|^2 
\end{multline}

Fix now a radius $r$ and set $g = \frac{r}{\|\widehat f - f_\star\|_{L^2}}(\widehat f-f_\star)$. If $\|\widehat f -f_\star\|_{L^2}\geq r  $, dividing both sides above by $\|\widehat f -f_\star\|_{L^2}$ yields for the first term above in \eqref{eq:thesplitismade}:
\begin{equation}
    \begin{aligned}
&\frac{2(1+\e)}{n\|\widehat f -f_\star\|_{L^2}}\sum_{i=1}^n  (1-\E')\langle W_i, (\widehat f-f_\star)(X_i)\rangle 
\\
        &=
        \frac{1}{ n}\sum_{i=1}^n\left\{ 2(1+\e)(1-\E') \langle W_i, r ^{-1} g(X_i) \rangle\right\} && (\textnormal{df. of $g$ and divide})\\
        &\leq r^{-1}\sup_{g\in \scrF_\star\cap r S_{L^2}} M_n(g).
    \end{aligned}
\end{equation}
Either the above inequality holds or $\|\widehat f-f_\star \|_{L^2} \leq r$. For every $r>0$ it is thus true that:
\begin{equation}\label{eq:offsetcomplexity}
    \begin{aligned}
        \|\widehat f-f_\star\|_{L^2}^2
        &
        \leq r^2 +\left(r^{-1}\sup_{g\in \scrF_\star\cap r S_{L^2}} M_n(g)\right)^2 + \sup_{g\in \scrF_\star} Q_n(g)
    \end{aligned}
\end{equation}
This proves the claim.
\end{proof}

\paragraph{Finishing the proof of \Cref{thm:themainthm}.}
    We apply \Cref{lem:basicineq} with $r=r_\star$, $\e =1/2$ and note that $  n \geq c_3 \max\left\{n_{\mathsf{quad}}(r_\star),n_{\mathsf{mult}}(r_\star)\right\}$ implies: (1) in combination with \Cref{thm:loweruniformmix} that $ \sup_{g\in \scrF_\star} Q_n(g)\lesssim r_\star^2$; and (2) in combination with \Cref{prop:multipliermix} that $\left(\sup_{g\in \scrF_\star\cap r_\star S_{L^2}} M_n(g)\right)^2$ scales at most like the RHS of \eqref{eq:themainthmeq}. The result follows by a union bound over the failure events of \Cref{prop:multipliermix}  and \Cref{thm:loweruniformmix}, all the while taking into account the fact that we posit $k\geq k_{\mathsf{mix}}$. \hfill $\blacksquare$

\section{Proof of the Corollaries to \Cref{thm:themainthm}}

\subsection{Proof of \Cref{cor:params}}

\paramclass*

\begin{proof}

Let us begin by observing that for  some constant $c_\eta$ only depending on $\eta $ we have that:
\begin{equation}
\begin{aligned}
\label{eq:gammafunccalcparamclass}
    \gamma_\eta(\scrF_\star \cap rS_{L^2},d_{L^2})
    &\leq
    c_{\eta} \int_0^r \left(d_\scrF \log \left( \frac{1}{s}\right)\right)^{1/\eta} ds\\
    &=
    c_{\eta} d_\scrF^{1/\eta} r  \int_0^1 \left( \log \left( \frac{r}{s}\right)\right)^{1/\eta} ds\\
    &\leq  c_{\eta} d_\scrF^{1/\eta} r \Gamma(1/\eta+1).
\end{aligned}
\end{equation}

Hence for some universal positive constant $c$:
\begin{equation}
     \sqrt{   \V_{}\left(\scrF_{\star}\cap rS_{L^2}\right) }  \times \frac{1}{r\sqrt{n}} \gamma_2(\scrF_{\star}\cap r S_{L^2},d_{L^2})
\\
\leq 
 c
 \sqrt{   \V_{}\left(\scrF_{\star}\cap rS_{L^2}\right) }  \times \frac{1}{\sqrt{n}}  d_\scrF^{1/2}.
\end{equation}
A few applications of the Cauchy-Schwarz inequality now yields for any $r$:
\begin{equation}
    \V\left(\scrF_{\star}\cap r S_{L^2}\right) \leq k \|W\|_{L^2}^2.
\end{equation}

A candidate choice is therefore $r_\star = \sqrt{\frac{d_\scrF \V\left(\scrF_{\star}\cap \sqrt{\frac{d_\scrF k \|W\|_{L^2}^2}{n}}S_{L^2}\right)}{n}}$. A straighforward but tedious calculation now reveals that the inequality $n \geq \max\left\{n_{\mathsf{quad}}(r_\star),n_{\mathsf{mult}}(r_\star)\right\}$ has a solution depending polynomially on problem data as long as $\eta >1/4$. 
\end{proof}

\subsection{Proof of \Cref{cor:linreg}}

\linreg*

\begin{proof}
We apply \Cref{thm:themainthm} with $p=\infty$, $q=1$ and $\eta=1$.  As in the proof of the preceding corollary (see \eqref{eq:gammafunccalcparamclass}), notice that
    \begin{equation}
    \label{eq:gammafunclinreg}
    \begin{aligned}
        \gamma_1(\scrF_\star \cap rS_{L^2},d_{L^2}) &\leq  c d r,  &\textnormal
        {and} \\
        \gamma_2(\scrF_\star \cap rS_{L^2},d_{L^2}) &\leq  c \sqrt{d} r.
    \end{aligned}
    \end{equation}
    Moreover, since $W_{1:n}$ is a martingale difference sequence we have $\V(\scrF_\star \cap r S_{L^2}) = \frac{1}{n}\sum_{i=1}^n \V(W_i)$. Consequently, the critical radius inequality \eqref{eq_thm:critrad} becomes
    \begin{equation*}
        r \geq c \times \sqrt{\frac{1}{n}\sum_{i=1}^n \V(W_i) \times \frac{d}{n}}
    \end{equation*}
    so that (using stationarity) we may choose $r_\star \propto \sqrt{\V(W) \times \frac{d}{n}} $.

    Let us now turn to evaluating \eqref{eq:burnindef} for this model. $n_\mathsf{quad}$ reads:
    \begin{equation*}
        \begin{aligned}
               n_{\mathsf{quad}}(r_\star)&= \inf \Bigg\{ n\in \N \Bigg| \Bigg [n^{-1/2}\sqrt{k}    L^{1+3/4}r_\star\times   
        \left( r_\star \sqrt{d}+ r_\star \sqrt{\log (1/\delta)}\right)
     \\
     &+
      n^{-1}   L^2  k          r_\star                   
      \left( r_\star d
     +r_\star \log(1/\delta) \right) \Bigg] \leq r_\star^2\Bigg\}\\
     &\leq
     \inf \Bigg\{ n\in \N \Bigg| \Bigg [n^{-1/2}\sqrt{k}    L^{1+3/4}\times   
        \left( \sqrt{d}+  \sqrt{\log (1/\delta)}\right) \Bigg] \leq 1 \Bigg\}\\
        &+
        \inf \Bigg\{ n\in \N \Bigg| \Bigg [
      n^{-1}   L^2  k                             
      \left(d
     + \log(1/\delta) \right) \Bigg] \leq1\Bigg\}\\
     &
     \leq  2k (L\vee 1)^{3+1/2} \left( d+\log(1/\delta) \right).
        \end{aligned}
    \end{equation*}
    
    Next, we turn to $n_\mathsf{mult}$:
    \begin{equation*}
    \begin{aligned}
          n_{\mathsf{mult}}(r_\star)&=\inf \Big\{ n \in \N\Big|  L  k B_W
      \left(\frac{d+ \log(1/\delta)}{n} \right) \leq \sqrt{\V(W)d/n }\Big\} 
      &\leq
           L^2k^2\frac{B_W^2}{\V(W)}\left( d + \log(1/\delta)\right).
    \end{aligned}
    \end{equation*}

    Moreover, it is easy to see that may choose $L = B_X / \sqrt{\lambda_{\mathrm{min}} (\E XX^\T)} \geq 1$. Hence the desired result follows under the burn-in requirement that:
    \begin{equation*}
        \frac{n}{k} \geq c
        \left( B_X / \sqrt{\lambda_{\mathrm{min}} (\E XX^\T)}\right)^{3+1/2}\left(\frac{kB_W^2}{\V(W)}\right)(d+ \log(1/\delta)) 
    \end{equation*}
    as we sought to prove.
\end{proof}

\end{document}